\documentclass[]{bytedance_seed}

\usepackage[toc,page,header]{appendix}
\usepackage{pgfplots}
\usepackage{times} 
\pgfplotsset{compat=1.18} 

\usepackage[table]{xcolor}
\definecolor{myblue}{HTML}{4A90E2}
\definecolor{myorange}{HTML}{F5A623}
\definecolor{mygreen}{HTML}{35F5C3}

\usepackage{minitoc}

\usepackage{booktabs} 
\usepackage{graphicx} 
\usepackage{arydshln} 
\usepackage{microtype}

\usepackage{algorithm}
\usepackage{algorithmic}
\usepackage{multirow}
\usepackage{amsmath}
\usepackage{amssymb}
\usepackage{amsthm} 
\usepackage{amsfonts}
\usepackage{placeins}

\usepackage{listings}
\usepackage{setspace}
\definecolor{codegray}{rgb}{0.5,0.5,0.5}
\definecolor{codepurple}{rgb}{0.58,0,0.82}
\definecolor{codeblue}{rgb}{0,0,0.6}
\definecolor{codegreen}{rgb}{0.1,0.5,0.1}

\lstdefinestyle{compactstyle}{
    backgroundcolor=\color{white},   
    commentstyle=\color{codegreen},
    keywordstyle=\color{codeblue}\bfseries,
    numberstyle=\tiny\color{codegray},
    stringstyle=\color{codepurple},
    basicstyle=\ttfamily\footnotesize, 
    breakatwhitespace=false,         
    breaklines=true,                 
    captionpos=b,                    
    keepspaces=true,                 
    numbers=left,                    
    numbersep=4pt,                   
    showspaces=false,                
    showstringspaces=false,
    showtabs=false,
    frame=single,                    
    rulecolor=\color{black!30},
    framerule=0.5pt,
    framesep=3pt,                    
    xleftmargin=2mm,                 
}

\usepackage{lipsum} 

\newtheorem{proposition}{Proposition}

\newcommand{\gain}[1]{\,\textcolor{teal}{(+#1)}} 

\newcommand{\dq}[1]{``#1''}

\theoremstyle{definition}
\newtheorem*{definition*}{Definition}

\newcommand{\E}{\mathbb{E}}
\newcommand{\gH}{g(H_t^{(\tau)})}
\newcommand{\fH}{f(H_{t+1}^{(\tau)})}
\newcommand{\grad}{\nabla_\theta}
\newcommand{\logpi}{\log \pi_\theta(a_t|s_t)}

\title{Harnessing Uncertainty: Entropy-Modulated Policy Gradients for Long-Horizon LLM Agents}

\author[\S]{Jiawei Wang}
\author[\S]{Jiacai Liu}
\author[\S]{Yuqian Fu}
\author[]{Yingru Li}
\author[\S]{Xintao Wang}
\author[\S]{Yuan Lin}
\author[\S]{Yu Yue}
\author[]{Lin Zhang}
\author[\S,\dagger]{Yang Wang}
\author[\S,\dagger]{Ke Wang}

\affiliation[]{ByteDance}

\contribution[\S]{Work done at ByteDance Seed}
\contribution[\dagger]{Corresponding authors}

\abstract{
In long-horizon tasks, recent agents based on Large Language Models (LLMs) face a significant challenge that sparse, outcome-based rewards make it difficult to assign credit to intermediate steps. Previous methods mainly focus on creating dense reward signals to guide learning, either through traditional reinforcement learning techniques like inverse reinforcement learning or by using Process Reward Models for step-by-step feedback.
In this paper, we identify a fundamental problem in the learning dynamics of LLMs: the magnitude of policy gradients is inherently coupled with the entropy, which leads to inefficient small updates for confident correct actions and potentially destabilizes large updates for uncertain ones. To resolve this, we propose Entropy-Modulated Policy Gradients (EMPG), a framework that re-calibrates the learning signal based on step-wise uncertainty and the final task outcome. EMPG amplifies updates for confident correct actions, penalizes confident errors, and attenuates updates from uncertain steps to stabilize exploration. We further introduce a bonus term for future clarity that encourages agents to find more predictable solution paths. Through comprehensive experiments on three challenging agent tasks, WebShop, ALFWorld, and Deep Search, we demonstrate that EMPG achieves substantial performance gains and significantly outperforms strong policy gradient baselines.
}

\date{\today}
\correspondence{Yang Wang at \email{wangyang.127@bytedance.com}, Ke Wang at \email{wangke@bytedance.com}}

\checkdata[Project Page]{\url{https://empgseed-seed.github.io/}}

\begin{document}
\maketitle


\section{Introduction}

The advent of Large Language Models (LLMs) has catalyzed the development of autonomous agents that are capable of tackling complex, multi-step tasks \citep{wei2022chain, yao2023react}. However, a fundamental challenge persists in training these agents for long-horizon tasks: the sparsity of outcome-based rewards. 
In many realistic scenarios, such as web navigation \citep{yao2022webshop}, software engineering \cite{zhang2024codeagent}, and deep search \citep{alzubi2025open}, feedback is only available at the end of the complete generation. This makes it difficult to assign appropriate credit for standard reinforcement learning (RL) algorithms to discern the crucial intermediate steps.

To solve the problem of the sparse reward challenge, prior work has explored two primary directions: implicit reward guidance and explicit step-wise supervision. The first involves traditional reinforcement learning techniques aimed at creating densified reward signals. Methods like reward shaping \citep{ng1999policy}, intrinsic motivation based on state novelty or curiosity \citep{bellemare2016unifying, pathak2017curiosity}, and inverse reinforcement learning \citep{ziebart2008maximum,deng2024novice} attempt to estimate the value of intermediate actions. However, these approaches often struggle to scale. They are either computationally prohibitive, ill-suited for the vast, combinatorially complex state and action spaces inherent to LLM-driven agent tasks, or heavily reliant on human prior knowledge. The second line of research, particularly successful in structured reasoning domains, employs Process Reward Models (PRMs) \citep{lightman2023let} to provide step-by-step feedback. Yet, PRMs suffer from significant drawbacks: they demand prohibitive human annotation costs to build, are susceptible to noise when trained on synthetic data, and often exhibit poor generalization to out-of-distribution problems. These limitations are exacerbated in complex, interactive agent tasks where defining a single "correct" step is itself a non-trivial, context-dependent challenge, making the application of PRMs impractical.

Policy entropy is a cornerstone concept in RL, traditionally used to balance the exploration-exploitation trade-off. Recently, it has been repurposed as a direct learning signal in LLM reasoning tasks, where minimizing entropy is used as an unsupervised objective to increase the model's certainty \citep{gao2025one, agarwal2025unreasonable}. While effective in some contexts, this approach is vulnerable to the critical issue of "hallucinated confidence," where the model becomes confidently incorrect \cite{zhang2025siren}. More recent efforts use entropy not as a reward, but as a modulator. For instance, Seed-GRPO \citep{chen2025seed} leverages semantic uncertainty to down-weight the advantage of high-entropy responses in mathematical reasoning, while \citet{cheng2025reasoning} proposes shaping the advantage function with token-level entropy to improve long-form generation. However, these efforts are restricted to single-turn, generative reasoning tasks, operating at either the token-level or response-level. 
It remains underexplored how to leverage agents' intrinsic uncertainty for credit assignment in long-horizon, multi-step decision-making.

\begin{figure}[t]
\centering
\includegraphics[width=1.0\textwidth]{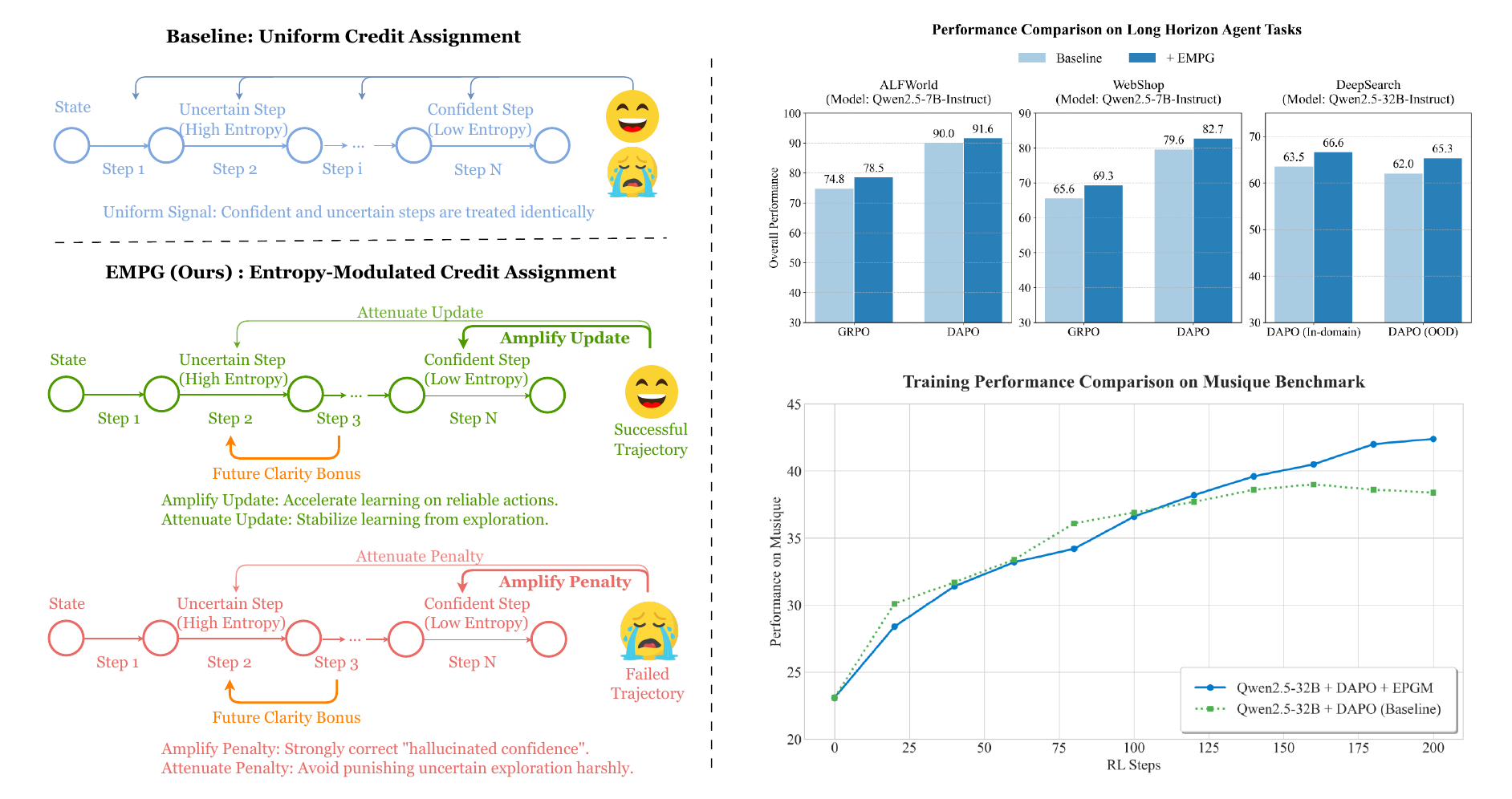}
\caption{Overview of the EMPG mechanism and its algorithm performance. \textbf{Left:} Conceptual diagram contrasting the uniform credit assignment of baseline methods with EMPG's confidence-modulated signal. \textbf{Right:} Final performance comparison on key long-horizon benchmarks showing EMPG's superiority, along with the training dynamics on Musique that highlight its ability to achieve sustained improvement and avoid the baseline's performance plateau.}
\label{fig:overview}
\end{figure}

Our work begins by analyzing the fundamental dynamics of the policy gradient itself. We formally show that for a standard softmax policy, the expected norm of the score function is a monotonic function of the policy's entropy (Proposition \ref{prop:expected_grad_norm}). In simple terms, high-entropy (uncertain) actions naturally produce large gradients, while low-entropy (confident) actions produce small ones. This inherent behavior presents a dual challenge for learning: 1) confident and correct steps, which should be strongly reinforced, receive small updates, limiting learning speed, and 2) uncertain exploratory steps can introduce large, noisy gradients that destabilize training. This reveals a critical need to explicitly re-calibrate the learning signal based on an action's uncertainty.

To address this, we propose Entropy-Modulated Policy Gradients (EMPG), a framework that reshapes the learning landscape by directly adapting to this dynamic, as illustrated in Figure \ref{fig:overview}. Instead of naively rewarding low entropy, EMPG introduces \textit{Self-Calibrating Gradient Scaling} mechanism, which dynamically modulates the policy gradient based on step-wise uncertainty: \textit{1) for confident and correct actions}, it amplifies the updates, while \textit{2) for uncertain steps}, it attenuates updates to ensure stable exploration.
Furthermore, to encourage agents to find predictable solution paths, EMPG introduces \textit{\dq{future clarity}}, an additional bonus term in the advantage function that provides an intrinsic signal for actions that lead to less uncertain subsequent states. 
This guides agents to perform purposeful exploration, steering them away from chaotic or unpromising high-entropy trajectories toward states with greater clarity about the next steps. This dual approach enables EMPG to forge a dense, informative, and well-calibrated learning signal from sparse external feedback. To validate our framework, we conduct extensive experiments on challenging long-horizon agent benchmarks such as WebShop \cite{yao2022webshop}, ALFWorld \cite{shridhar2021alfworld}, and Deep Search \cite{alzubi2025open}, demonstrating the effectiveness and scalability of our approach across models of various sizes.


Our key contributions are as follows:
\begin{itemize}
    \item We first identify and formalize a fundamental challenge in policy gradient methods: the inherent coupling of gradient magnitude and policy entropy. This dynamic leads to inefficient learning for confident actions and instability from uncertain ones, motivating the need for explicit signal re-calibration.
    
    \item We introduce Entropy-Modulated Policy Gradients, a framework designed to solve this problem. EMPG combines \textit{Self-Calibrating Gradient Scaling} to correct the flawed gradient dynamics with a \textit{Future Clarity Bonus} to promote exploration towards more predictable states.
    
    \item Extensive experiments on demanding agent tasks (WebShop, ALFWorld, Deep Search) show that EMPG substantially outperforms strong baselines like GRPO and DAPO.
\end{itemize}
\section{Related Work}

\subsection{LLM-based Autonomous Agents}

The advent of LLMs has catalyzed the development of sophisticated autonomous agents capable of performing complex, multi-step tasks that were previously unattainable. 
Specialized agents have been designed for diverse applications, including software development (e.g., coding agents \cite{jimenez2023swe,zhang2024codeagent}), information retrieval (search agents \cite{he2025pasa,li2025search}), and complex web interactions (browser-use agents \cite{yao2022webshop,deng2023mind2web,yan2023gpt}). For training these agentic models, reinforcement learning has proven to be a powerful and essential paradigm. Recent research on RL-based agents, such as Search-R1 \cite{jin2025search}, SWE-RL \cite{wei2025swe}, and WebAgent-R1 \cite{wei2025webagent}, has demonstrated that RL can effectively enhance agent performance and enable learning in highly interactive and dynamic environments. Despite these successes, a fundamental problem remains to be fully addressed: the difficulty of credit assignment in long-horizon tasks. The multi-step nature of these problems, where a reward signal is often only available upon completion, hinders the efficiency and stability of the training process.

\subsection{Reinforcement Learning from Internal Feedback}

To overcome the challenges of sparse external rewards, recent studies have explored using internal feedback, generated by the model itself, to create denser training signals. This approach often leverages unsupervised signals derived from model uncertainty \cite{zhang2025right,agarwal2025unreasonable,zhao2025learning} or self-consistency \cite{zuo2025ttrl,zhang2025consistent}, frequently quantified by policy entropy. However, the role of entropy has been interpreted in conflicting ways. Some studies argue that correct responses typically exhibit lower entropy, thus proposing unsupervised entropy minimization as a method to improve performance \cite{gao2025one,agarwal2025unreasonable}. Conversely, other works suggest that high entropy encourages exploratory reasoning. For instance, SEED-GRPO \cite{chen2025seed} uses semantic entropy to modulate policy updates for diversity, while others explicitly incorporate policy entropy into the advantage term to promote exploration \citep{cheng2025reasoning,vanlioglu2025entropy}.
Recently, EDGE-GRPO \citep{zhang2025edge} proposes entropy modulation in single-turn mathematical reasoning.
Similar to our method, they modulate policy gradients by amplifying updates for confident correct responses and attenuating updates for incorrect or uncertain ones.  
However, EMPG fundamentally differs from EDGE-GRPO in both motivation and scope:
First, while EDGE-GRPO focuses on correcting confidence misalignment within a single-turn mathematical reasoning,
EMPG is specifically designed for the multi-step credit assignment problem in long-horizon tasks. Second, towards the challenges in multi-turn long-horizon tasks, EMPG dynamically assigns credit across the entire trajectory to amplify the crucial steps.

\clearpage
\section{Preliminaries}
\label{sec:Preliminaries}

\subsection{Policy Optimization in Reinforcement Learning}
Our work is grounded in policy gradient methods, which seek to optimize a policy $\pi_\theta$ parameterized by $\theta$ to maximize the expected reward objective:
\begin{equation}
    \mathcal{J}(\pi_\theta) := \mathbb{E}_{\tau \sim \pi_\theta}[R(\tau)]
\end{equation}
where $\tau$ is a trajectory sampled under policy $\pi_\theta$ and $R(\tau)$ is its total return. The policy gradient theorem allows for direct optimization of this objective via gradient ascent. The gradient is estimated as an expectation over trajectories:
\begin{equation}
    \nabla_\theta \mathcal{J}(\pi_\theta) = \mathbb{E}_{\tau \sim \pi_\theta} \left[ \sum_{t=0}^{T} A(s_t, a_t) \nabla_\theta \log \pi_\theta(a_t|s_t) \right]
\end{equation}
where $s_t$ and $a_t$ are the state and action at time step t, respectively.

A key challenge in estimating this gradient is its inherently high variance. To mitigate this, an advantage function, $A(s_t, a_t)$, is used to measure the relative quality of an action. This advantage is typically estimated using a learned value model, which predicts the expected return from a given state \citep{schulman2017proximal}. However, this approach has significant drawbacks. The value model is often comparable in size to the policy model, introducing substantial memory and computational overhead. Furthermore, the effectiveness of the algorithm hinges on the reliability of its value estimates, which are inherently difficult to learn accurately \cite{liu2024improving,kazemnejad2024vineppo}, especially for complex tasks with long response horizons. Due to these challenges, value-free methods, which estimate the advantage directly from sampled trajectories without a learned value function, have become increasingly popular \citep{shao2024deepseekmath, yu2025dapo}. Our work is also grounded in this value-free paradigm, foregoing a value model to improve training efficiency and stability.

\subsection{RL Framework for Long-Horizon Agent Tasks}
We formalize the long-horizon task as a standard reinforcement learning problem. An LLM agent interacts with an environment over a trajectory $\tau = (s_0, a_0, r_0, ..., s_T, a_T, r_T)$. The reward signal is sparse, with $r_t=0$ for all non-terminal steps. Assuming an undiscounted setting ($\gamma=1$), the trajectory return $R(\tau)$ is thus determined solely by the final outcome:
\begin{equation}
    R(\tau) = \sum_{t=0}^{T} \gamma^t r_t = r_T \in \{0, 1\}
\end{equation}
In our work, a single step corresponds to a complete "reason-then-act" cycle (e.g., as in ReAct \citep{yao2023react}), forming a multi-step decision-making process. This sparse-reward, long-horizon setting epitomizes two fundamental RL challenges: the \textbf{credit assignment problem} and the \textbf{exploration problem}.

\subsection{Strategies for Learning from Sparse Outcome-Based Rewards}
To enable effective learning from sparse, outcome-based rewards in long-horizon tasks, several powerful strategies have emerged that form the foundation of modern LLM RL.

\begin{itemize}
    \item \textbf{Trust Region Learning.} Proximal Policy Optimization (PPO) \citep{schulman2017proximal} serves as the bedrock algorithm. Its primary innovation is not credit assignment, but ensuring training stability. It achieves this by constraining policy updates within a trust region, using a clipped objective on the probability ratio $\rho_t(\theta) = \frac{\pi_\theta(a_t|s_t)}{\pi_{\theta_{\text{old}}}(a_t|s_t)}$. When applied to sparse reward tasks, PPO's effectiveness fundamentally depends on the quality of its advantage estimates, which implicitly perform the task of credit assignment \cite{kazemnejad2024vineppo}.

    \item \textbf{Group-Based Advantage Estimation.} Group Relative Policy Optimization (GRPO) \citep{shao2024deepseekmath} builds upon this foundation with a direct solution for credit assignment. It addresses the high variance of the policy gradient inherent in sparse rewards by sampling multiple responses ($M$) and computing a Z-score-like advantage:
    \begin{equation}
        A_{ij} = \frac{r(x_i, y_{ij}) - \text{mean}_{k=1}^M(r(x_i, y_{ik}))}{\text{std}_{k=1}^M(r(x_i, y_{ik})) + \epsilon}
    \end{equation}
    Here, $r(x_i, y_{ij})$ is the final outcome-based reward for the j-th response, and $\epsilon$ is a small constant added for numerical stability. This comparative evaluation effectively identifies the best-in-batch responses, providing a robust signal.
    
    \item \textbf{Adaptive Data Curation.} Decoupled Clip and Dynamic Sampling Policy Optimization (DAPO) \citep{yu2025dapo} further refines the learning process by curating the data itself. It addresses failure modes in GRPO by filtering and resampling trajectories to form more informative training batches. By focusing updates on a buffer of high-quality samples, it improves the efficiency of learning from the sparse reward signal.
\end{itemize}

While powerful, these strategies share a common reliance on processing external, outcome-based reward signals. As they are primarily designed for single-turn generation, they treat entire action sequences as monolithic blocks. When applied to interactive agent tasks, this leads to a coarse, trajectory-level credit assignment that fails to pinpoint which specific actions in a long sequence were critical for success. This approach ignores the rich, intrinsic signals available at each step of the generative process. Our work diverges by proposing a new paradigm that peers inside the model, leveraging its intrinsic, step-wise uncertainty.

\subsection{Theoretical Motivation: A Two-Part Re-Calibration of Policy Gradients}
\label{sec:theoretical_motivation}
Our approach is motivated by a fundamental analysis of the relationship between a policy's gradient and its predictive uncertainty. Standard policy gradients, while effective, possess an inherent dynamic that can hinder stable and efficient learning. Specifically, the magnitude of the gradient is inherently coupled with the policy's entropy, often leading to inefficiently small updates for confident actions and potentially destabilizing large updates for uncertain ones. This dynamic, which we aim to re-calibrate, is formally characterized by the following proposition.

\begin{proposition}
\label{prop:expected_grad_norm}
For a policy $\pi_\theta$ parameterized by a softmax over logits $z_\theta(s)$, the expected squared L2-norm of the score function $\nabla_{z_\theta} \log \pi_\theta(a|s)$ with respect to the logits is a direct function of the policy's Rényi-2 entropy \cite{renyi1961measures}, $H_2(\pi)$:
\begin{equation}
\label{eq:grad_norm_entropy}
    \mathbb{E}_{a \sim \pi_\theta(\cdot|s)} \left[ ||\nabla_{z_\theta (s)} \log \pi_\theta(a|s)||^2 \right] = 1 - \exp(-H_2(\pi))
\end{equation}
A detailed proof is provided in Appendix \ref{app:proof}.
\end{proposition}

Equation \eqref{eq:grad_norm_entropy}, which builds upon established relationships between different measures of policy entropy (e.g., in \citet{li2025logit}), proves that the expected gradient norm is monotonically coupled with policy entropy. This presents a dual challenge: \textit{1)} a confident and correct step should be reinforced strongly, but its naturally small gradient limits its impact; and \textit{2)} the large gradients from highly uncertain exploratory steps can introduce noise and destabilize training. Our first component, \textit{Self-Calibrating Gradient Scaling}, directly addresses this by re-calibrating the \textit{magnitude} of the update based on current-step uncertainty.

However, re-calibrating the update magnitude is only half the solution. A truly effective learning signal must also guide the agent in a useful \textit{direction}. This motivates our second component, the \textit{Future Clarity Bonus}, which can be conceptually justified through the lens of information theory. By providing an intrinsic motivation for the agent to seek low-entropy next states, the bonus encourages actions that yield high \textit{Information Gain} about the optimal future path. This corresponds to a local, step-wise objective of minimizing the policy's entropy at the next state:
\begin{equation}
    \label{eq:fcb_objective}
    \min_{a_t} H\!\bigl(\pi_{\theta}(s_{t+1})\bigr).
\end{equation}
This objective, which aligns with established principles like the Empowerment framework \cite{klyubin2005empowerment}, imbues the agent with a generalizable meta-skill: to actively seek clarity in the face of ambiguity.

In summary, EMPG provides a complete, two-part re-calibration of the learning signal. The gradient scaling module ensures each update has an appropriate \textit{magnitude}, while the future clarity bonus provides a principled intrinsic motivation that shapes the policy's \textit{direction} towards robust and predictable solution paths.
\section{Entropy-Modulated Policy Gradients}
\label{sec:approach}

Building on the theoretical motivation established in our preliminaries, we introduce Entropy-Modulated Policy Gradients (EMPG), a framework designed to re-calibrate the learning dynamics of policy gradients for long-horizon agent tasks. As shown in Section \ref{sec:theoretical_motivation}, standard policy gradients are inherently biased towards applying smaller updates to confident (low-entropy) steps and larger updates to uncertain (high-entropy) ones. EMPG is engineered to counteract this behavior, enabling more efficient and stable learning from sparse, outcome-based rewards.

\subsection{Quantifying Step-Level Uncertainty}
The core of our method is to quantify the agent's confidence at each decision-making step. While various uncertainty measures exist, we opt for a practical and computationally efficient proxy: the average token-level entropy over a single "reason-then-act" step. For a step $step_t$ composed of tokens $\{w_1, ..., w_m\}$, the step-level entropy $H_t$ is:
\begin{equation}
    H_t = -\frac{1}{m} \sum_{j=1}^{m} \sum_{v \in V} p(v|w_{<j}) \log p(v|w_{<j})
\end{equation}
where $p(v|w_{<j})$ is the probability of token $v$ from the vocabulary $V$, as provided by the LLM's policy $\pi_\theta$. A lower $H_t$ indicates higher confidence in the generated step, corresponding to a lower-entropy state in the sense of Proposition \ref{prop:expected_grad_norm}.

While we use policy entropy for its computational efficiency, future work could explore alternative uncertainty estimators, such as those derived from Monte Carlo dropout or the variance in logits from an ensemble of model heads. However, we believe entropy provides the most direct link to the gradient dynamics analyzed in Proposition \ref{prop:expected_grad_norm}, making it the most theoretically grounded choice for our framework.

\subsection{The Modulated Advantage for Gradient Re-Calibrating}
In the sparse reward setting, a standard RL advantage function provides a uniform learning signal for all steps within a single trajectory. While simple, this approach overlooks the varying contributions of different steps and their impact on learning stability. To address this, we introduce a novel, modulated advantage estimate, $A_{\text{mod}}$, for each step $t$ in a trajectory $\tau_i$:
\begin{equation}
\label{eq:main_formula}
A_{\text{mod}}(i, t) = \underbrace{A^{(i)} \cdot g(H_t^{(i)})}_{\text{self-calibrating gradient scaling}} + \underbrace{\zeta \cdot f(H_{t+1}^{(i)})}_{\text{future clarity bonus}}
\end{equation}
This formulation fundamentally re-calibrates the learning signal through two complementary forms of \textbf{advantage shaping}. The first term utilizes a step-level entropy-based function $g(H_t^{(i)})$ to dynamically reweight the trajectory's shared advantage $A^{(i)}$, thereby achieving a more granular and confidence-aware gradient update. The second term, a \textbf{future clarity bonus}, is an additive shaping signal that encourages the agent to select actions that lead to a more predictable and less ambiguous future state. Together, these two mechanisms transform a coarse, trajectory-level signal into a rich and precise learning signal for each step, which we analyze further in the following sections.

\paragraph{Self-Calibrating Gradient Scaling $g(H)$.}
To counteract the natural gradient dynamics, the scaling function $g(H)$ is designed to be self-calibrating and adaptive. It achieves this by enforcing the constraint that the mean of $g(H_t^{(i)})$ over any given mini-batch is normalized to one. Mathematically, for a mini-batch of size $N_B$, this constraint is given by:
\begin{equation}
\label{eq:g_constraint}
\frac{1}{\sum_{i=1}^{N_B}T_i} \sum_{i=1}^{N_B} \sum_{t=1}^{T_{i}} g(H_t^{(i)}) = 1
\end{equation}
This principled design ensures the modulation redistributes the learning signal rather than simply inflating or deflating it, offering stability, adaptivity, and a reduction in hyperparameters. We implement this by normalizing a base exponential function by its mean over the mini-batch:
\begin{equation}
\label{eq:g_func}
g(H_t^{(i)}) = \frac{\exp(-k \cdot H_{\text{norm}, t}^{(i)})}{\frac{1}{\sum_{j=1}^{N_B}T_j} \sum_{j=1}^{N_B} \sum_{t'=1}^{T_{j}} \exp(-k \cdot H_{\text{norm}, t'}^{(i)})}
\end{equation}
For a confident step ($H_t^{(i)}$ is lower than the batch average), $g(H_t^{(i)}) > 1$, which \textbf{amplifies} its gradient. This accelerates convergence for confident and correct decisions ($A^{(i)}>0$) and provides a strong corrective penalty for confident errors ($A^{(i)}<0$), combating "hallucinated confidence". Conversely, for an uncertain step ($H_t^{(i)}$ is higher than average), $g(H_t^{(i)}) < 1$, which \textbf{attenuates} its gradient, preventing noisy updates from high-entropy exploration from destabilizing the policy.

\paragraph{Future Clarity Bonus $f(H)$.}
Beyond re-calibrating individual step updates, EMPG also encourages the agent to find globally stable and predictable solution paths. The second term in Eq. \ref{eq:main_formula} serves as an intrinsic motivation for this goal:
\begin{equation}
\label{eq:g_prime_func}
f(H_{t+1}^{(i)}) = \exp(-k' \cdot H_{\text{norm}, t+1}^{(i)})
\end{equation}
This term adds a positive bonus proportional to the confidence (low entropy) of the \textbf{next} step. Weighted by the hyperparameter $\zeta > 0$, this "future clarity" bonus actively guides the agent away from states of high confusion and towards sequences of high-quality, unambiguous decisions.

\subsection{Normalization Procedures}
\paragraph{Batch-Level Entropy Normalization.}
To ensure the modulation function $g(H)$ operates on a consistent scale, we normalize step-level entropies within each training batch using min-max scaling. This stateless approach allows the normalization to adapt dynamically to the policy's evolving confidence level. For each entropy value $H_t$ in the batch:
\begin{equation}
\label{eq:h_norm}
H_{\text{norm}, t}^{(i)} = \frac{H_t^{i} - \min_{\text{batch}}(H)}{\max_{\text{batch}}(H) - \min_{\text{batch}}(H) + \epsilon}
\end{equation}

\paragraph{Final Advantage Normalization.}
After computing the modulated advantage $A_{\text{mod}}$ for all steps in a batch, we perform a final batch-level normalization (zero mean). This standard variance reduction technique, which is crucial for stable policy updates, is achieved by subtracting the mean of $A_{\text{mod}}$ over the mini-batch of size $N_B$:
\begin{equation}
\label{eq:adv_norm}
A_{\text{final}}(i, t) = A_{\text{mod}}(i, t) - \frac{1}{N_B} \sum_{j=1}^{N_B}\sum_{t_j=1}^{T_j} A_{\text{mod}}(j, t_j)
\end{equation}

The overall EMPG algorithm is summarized in Algorithm \ref{alg:empg}, with an implementation provided in the appendix \ref{sec:appendix_implementation}. Furthermore, we provide a rigorous theoretical derivation for the EMPG update rule in Appendix~\ref{app:theoretical_derivation_empg}.

\begin{algorithm}[tb]
   \caption{Entropy-Modulated Policy Gradients (EMPG)}
   \label{alg:empg}
\begin{algorithmic}[1]
   \STATE {\bfseries Initialize:} Policy $\pi_\theta$.
   \FOR{each training iteration}
       \STATE Collect a batch of trajectories $\mathcal{B} = \{\tau_i\}$ by running policy $\pi_\theta$.
       \STATE Calculate outcome-based advantages $A^{(i)}$ for each trajectory $\tau_i \in \mathcal{B}$.
       \STATE Compute all step-level entropies $\{H_t\}$ for all steps in the batch.
       \STATE Normalize all entropies $\{H_t\}$ to $\{H_{\text{norm}, t}\}$ using batch min-max scaling.
       \STATE Compute the self-calibrating scaling factors $\{g(H_t)\}$ for all steps using Eq. \ref{eq:g_func}.
       \FOR{each step $t$ in each trajectory $\tau_i$}
           \STATE Calculate future clarity bonus $f(H_{t+1}^{(i)})$ using Eq. \ref{eq:g_prime_func}.
           \STATE Compute modulated advantage $A_{\text{mod}}(i, t)$ using Eq. \ref{eq:main_formula}.
       \ENDFOR
       \STATE Normalize the batch of all modulated advantages to get $\{A_{\text{final}}(i, t)\}$.
       \STATE Update policy parameters $\theta$ using policy gradients with $\{A_{\text{final}}(i, t)\}$.
   \ENDFOR
\end{algorithmic}
\end{algorithm}
\section{Experiments}

\subsection{Experimental Setup}

\paragraph{Tasks and Benchmarks.}
We evaluate our method on three challenging long-horizon agent benchmarks featuring sparse, binary success rewards: WebShop \citep{yao2022webshop}, a web navigation task requiring complex instruction following; ALFWorld \citep{shridhar2021alfworld}, a text-based environment combining instruction following with common-sense reasoning; and Deep Search \cite{jin2025search}, a multi-step information retrieval and synthesis task. For Deep Search, we further categorize the evaluation sets into in-domain (ID) and out-of-domain (OOD) to assess generalization.

\paragraph{Models and Agent Framework.}
Our agent employs the ReAct paradigm \citep{yao2023react}, where the LLM first generates a thought before producing an action. For WebShop and ALFWorld, we use Qwen2.5-1.5B-Instruct and Qwen2.5-7B-Instruct to compare our results with existing work. For the more complex Deep Search task, we use the powerful Qwen2.5-32B-Instruct model to conduct in-depth analysis.

\paragraph{Baselines and Implementation.}
We compare EMPG against strong policy gradient baselines: GRPO \citep{shao2024deepseekmath} and DAPO \citep{yu2025dapo}. Our method, EMPG, is implemented as an advantage modulation module that is applied directly on top of these baselines. This allows us to fairly measure the benefits of leveraging intrinsic uncertainty signals. For the WebShop and ALFWorld benchmarks, we based our implementation on the public codebase of GiGPO \citep{feng2025group} for a fair comparison. For the DeepSearch benchmark, we curated a training dataset of 17k instances by filtering from several sources, including WebWalker \citep{wu2025webwalker}, HotpotQA \citep{yang2018hotpotqa}, 2WikiMultiHopQA \citep{ho2020constructing}, NaturalQuestions \citep{kwiatkowski2019natural}, and TriviaQA \citep{joshi2017triviaqa}.

\subsection{Main Results}

Our comprehensive experiments demonstrate that EMPG yields significant and consistent performance improvements across a diverse range of tasks, baselines, and model scales.

\paragraph{Performance on ALFWorld and WebShop.}
As shown in Table~\ref{tab:main_result_webshop_alfworld}, EMPG serves as a robust enhancement to existing policy optimization algorithms. On the Qwen2.5-1.5B model, applying EMPG boosts the average success rate of GRPO on ALFWorld by +8.1 points and DAPO by +7.3 points. This effectiveness scales to the larger Qwen2.5-7B model, where EMPG again improves both baselines on ALFWorld and elevates the DAPO success rate on WebShop to an impressive 82.7\%. These results confirm that EMPG is highly compatible and provides reliable gains for different RL backbones.

\paragraph{Performance and Scalability on Deep Search.}
To investigate the scalability of our approach on more powerful models and complex retrieval tasks, we evaluated EMPG on the Deep Search benchmark using the Qwen2.5-32B-Instruct model. The results, presented in Table~\ref{tab:main_results_deepsearch_domain}, further validate our method. Applying EMPG to the strong DAPO baseline boosts the overall average score from 62.0 to 65.3, a substantial improvement of +3.3 points. This performance gain is notably robust, with EMPG improving the in-domain average by +3.1 points and demonstrating even stronger generalization with a +3.9 point gain on out-of-domain tasks.

\paragraph{Summary.}
Taken together, the results across all three benchmarks confirm that EMPG is a versatile and scalable enhancement for training LLM agents. It consistently improves performance regardless of the underlying RL algorithm, the nature of the task, or the size of the base model, validating our core hypothesis that leveraging intrinsic uncertainty is a powerful tool for learning from sparse rewards.

\begin{table}[t]
\centering
\caption{Performance on ALFWorld and WebShop. Results are averaged over 3 random seeds. For ALFWorld, we report the average success rate (\%) for each subtask as well as the overall result. For WebShop, we report both the average score and the average success rate (\%). Methods marked with * are our reproduced results. The remaining results are adopted from GiGPO~\cite{feng2025group}.}
\label{tab:main_result_webshop_alfworld}
\resizebox{0.8\textwidth}{!}{%
\begin{tabular}{@{}lccccccc|cc@{}}
\toprule
\multirow{2}{*}{Method} & \multicolumn{7}{c}{ALFWorld} & \multicolumn{2}{c}{WebShop} \\
\cmidrule(lr){2-8} \cmidrule(lr){9-10}
 & Pick & Look & Clean & Heat & Cool & Pick2 & All & Score & Succ. \\ \midrule
\multicolumn{10}{@{}l}{\textit{Base: Closed-Source Model}} \\
Prompting GPT-4o & 75.3 & 60.8 & 31.2 & 56.7 & 21.6 & 49.8 & 48.0 & 31.8 & 23.7 \\
Prompting Gemini-2.5-Pro & {92.8} & {63.3} & {62.1} & {69.0} & {26.6} & {58.7} & {60.3} & {42.5} & {35.9} \\ \midrule
\multicolumn{10}{@{}l}{\textit{Base: Qwen2.5-1.5B-Instruct}} \\
Prompting Qwen2.5 & 5.9 & 5.5 & 3.3 & 9.7 & 4.2 & 0.0 & 4.1 & 23.1 & 5.2 \\
Prompting ReAct & 17.4 & 20.5 & 15.7 & 6.2 & 7.7 & 2.0 & 12.8 & 40.1 & 11.3 \\
Prompting Reflexion & 35.3 & 22.2 & 21.7 & 13.6 & 19.4 & 3.7 & 21.8 & 55.8 & 21.9 \\
RL Training PPO (with critic) & 64.8 & 40.5 & 57.1 & 60.6 & 46.4 & 47.4 & 54.4 & 73.8 & 51.5 \\
RL Training RLOO & 88.3 & 52.8 & 71.0 & 62.8 & 66.4 & 56.9 & 69.7 & 73.9 & 52.1 \\
RL Training GRPO* & 87.9 & 40.0 & 78.1 & 35.7 & 65.2 & 44.4 & 65.6 & 78.0 & 58.2 \\
\quad\quad with EMPG* & 85.5 & 33.5 & 78.9 & 76.2 & 74.7 & 69.1 & 73.7\gain{8.1} & 80.4 & 60.8\gain{2.6} \\
RL Training DAPO* & 88.1 & 61.4 & 82.5 & \textbf{90.1} & 83.9 & 69.5 & 80.8 & 85.9 & 73.2 \\
\quad\quad with EMPG* & \textbf{97.7} & \textbf{80.7} & \textbf{87.5} & 87.0 & \textbf{88.3} & \textbf{80.0} & \textbf{88.1}\gain{7.3} & \textbf{86.8} & \textbf{73.8}\gain{0.6} \\
\midrule
\multicolumn{10}{@{}l}{\textit{Base: Qwen2.5-7B-Instruct}} \\
Prompting Qwen2.5 & 33.4 & 21.6 & 19.3 & 6.9 & 2.8 & 3.2 & 14.8 & 26.4 & 7.8 \\
Prompting ReAct & 48.5 & 35.4 & 34.3 & 13.2 & 18.2 & 17.6 & 31.2 & 46.2 & 19.5 \\
Prompting Reflexion & 62.0 & 41.6 & 44.9 & 30.9 & 36.3 & 23.8 & 42.7 & 58.1 & 28.8 \\
RL Training PPO (with critic) & 92.3 & 64.0 & 92.5 & 89.5 & 80.3 & 68.8 & 80.4 & 81.4 & 68.7 \\
RL Training RLOO & 87.6 & 78.2 & 87.3 & 81.3 & 71.9 & 48.9 & 75.5 & 80.3 & 65.7 \\
RL Training GRPO* & 88.8 & 43.7 & 88.1 & 70.3 & 77.7 & 56.8 & 74.8 & 77.8 & 65.6 \\
\quad\quad with EMPG* & 92.9 & 75.2 & 74.8 & 86.3 & 73.7 & 65.3 & 78.5\gain{3.7} & 81.0 & 69.3\gain{3.7} \\
RL Training DAPO* & 98.9 & 86.1 & 94.9 & 83.2 & \textbf{81.4} & 90.1 & 90.0 & 90.6 & 79.6 \\
\quad\quad with EMPG* & \textbf{99.0} & \textbf{86.8} & \textbf{97.3} & \textbf{94.9} & 75.8 & \textbf{90.3} & \textbf{91.6}\gain{1.6} & \textbf{92.0} & \textbf{82.7}\gain{3.1} \\
\bottomrule
\end{tabular}}
\end{table}


\begin{table}[t]
    \centering
    \caption{Main results on Deep Search tasks, categorized by domain. EMPG demonstrates strong performance on both in-domain (ID) and out-of-domain (OOD) datasets, with a particularly notable gain in generalization to OOD tasks.}
    \label{tab:main_results_deepsearch_domain}
    \resizebox{0.8\textwidth}{!}{%
    \begin{tabular}{l ccc c ccc c}
        \toprule
        & \multicolumn{4}{c}{\textbf{In-domain (ID)}} & \multicolumn{3}{c}{\textbf{Out-of-domain (OOD)}} & \textbf{Overall} \\
        \cmidrule(lr){2-5} \cmidrule(lr){6-8}
        \textbf{Method} & \textbf{WebWalker} & \textbf{HotpotQA} & \textbf{2wiki} & \textbf{Avg.} & \textbf{Musique} & \textbf{Bamboogle} & \textbf{Avg.} & \textbf{Avg.} \\
        \midrule
        \multicolumn{9}{l}{\textbf{Qwen2.5-32B-Instruct}} \\
        DAPO (Baseline) & 55.1 & 66.4 & 68.9 & 63.5 & 38.8 & 80.8 & 59.8 & 62.0 \\
        \midrule
        \multicolumn{9}{l}{\textit{Ablation Studies}} \\
        + Gradient Scaling & 54.9 & 68.8 & 67.4 & 63.7 & 41.0 & 86.4 & 63.7 & 63.7 \\
        + Future Bonus & \textbf{60.6} & 69.7 & 67.9 & 66.1 & 40.4 & 82.4 & 61.4 & 64.2 \\
        \midrule
        \textbf{+ EMPG (Ours)} & 57.5 & \textbf{71.2} & \textbf{71.0} & \textbf{66.6} & \textbf{41.8} & \textbf{84.8} & \textbf{63.7} & \textbf{65.3} \\
        \midrule
        \textbf{Gain vs. Baseline} & \gain{2.4} & \gain{4.8} & \gain{2.1} & \textbf{\gain{3.1}} & \gain{3.0} & \gain{4.0} & \textbf{\gain{3.9}} & \textbf{\gain{3.3}} \\
        \bottomrule
    \end{tabular}}
\end{table}

\subsection{Analysis}

To understand the mechanisms behind EMPG's effectiveness, we conduct a series of in-depth analyses focusing on three key questions: (1) What are the individual contributions of EMPG's core components? (2) How does EMPG affect the learning process over time? (3) Why is a step-level analysis of entropy crucial?
\begin{figure}[t]
    \begin{minipage}[b]{0.48\textwidth}
        \centering
        \includegraphics[width=\linewidth]{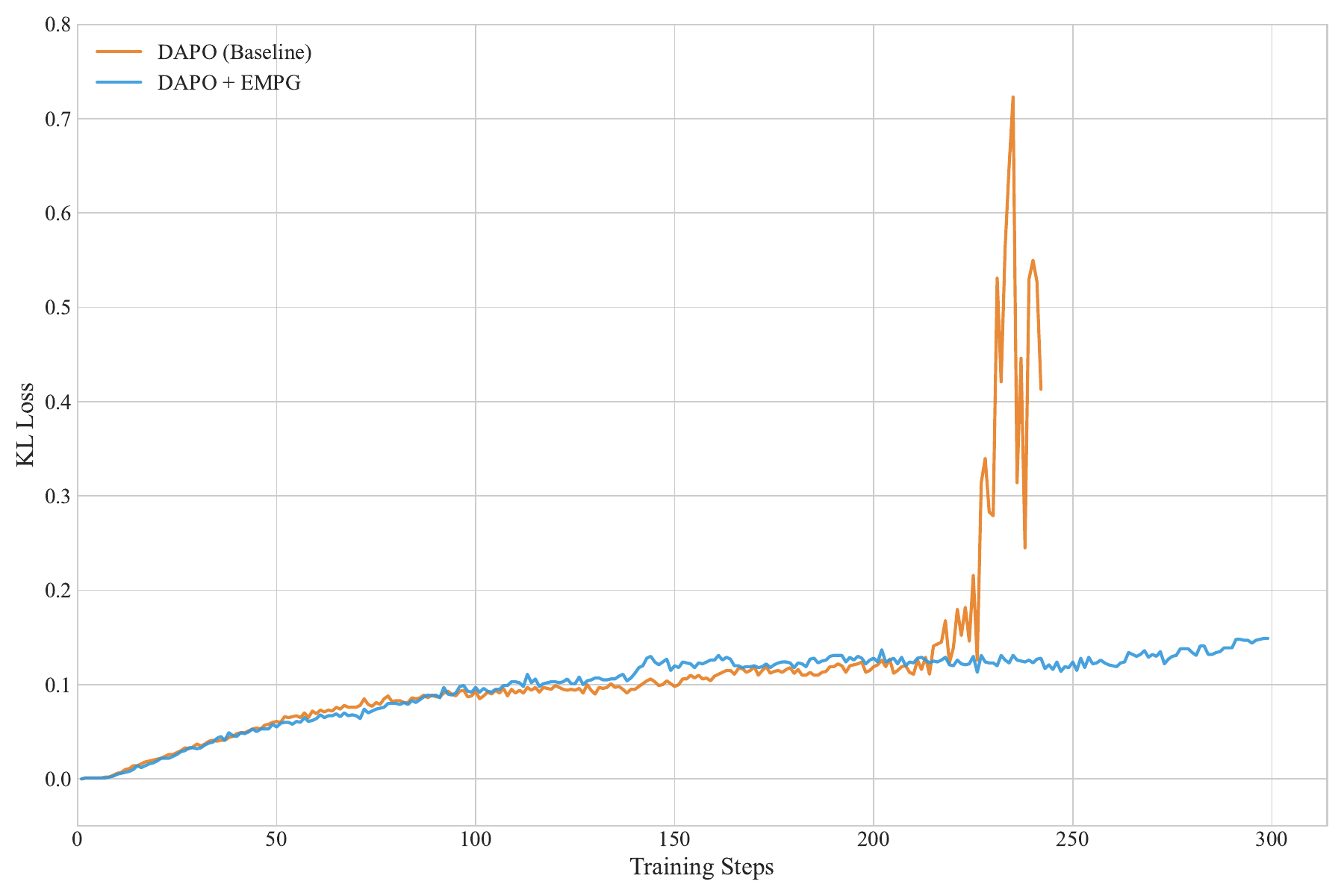}
        \caption{KL Loss dynamics during training for the Qwen2.5-32B-Instruct model. The DAPO baseline (orange) suffers from late-stage instability, evidenced by the sharp, erratic spike in KL Loss. The EMPG-enhanced model (blue) remains stable throughout, showcasing its robustness.}
        \label{fig:kl_loss_stability}
    \end{minipage}
    \hfill 
    \begin{minipage}[b]{0.48\textwidth}
        \centering
        \includegraphics[width=\linewidth]{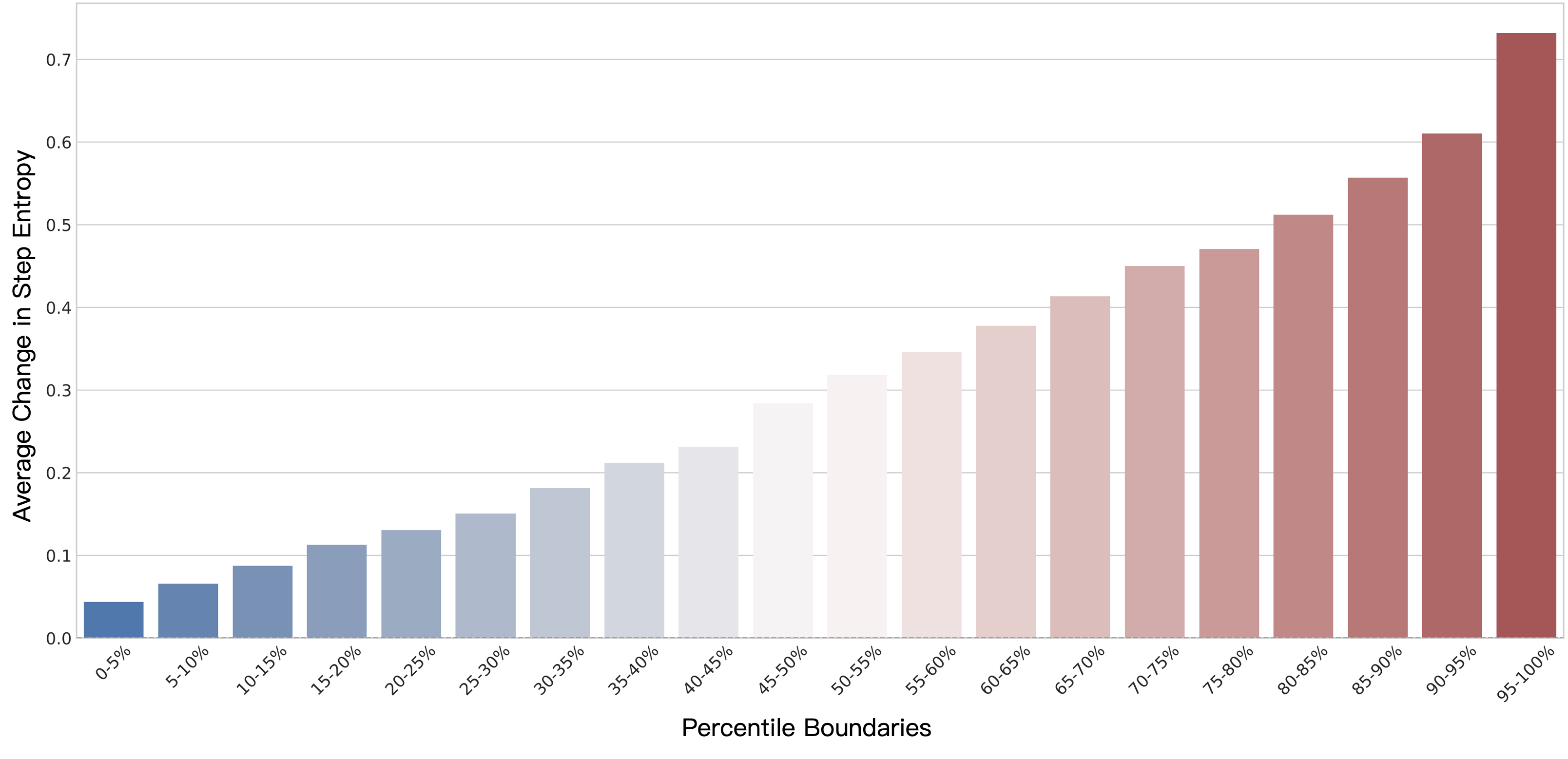}
        \caption{Average entropy change after RL fine-tuning within each 5\% entropy percentile range. Unlike token-level findings, even low-entropy steps undergo significant changes, validating our step-level analysis.}
        \label{fig:step_entropy_change}
    \end{minipage}
\end{figure}

\paragraph{Ablation Study and Generalization Analysis.}
To dissect the contributions of our method's two main components, we perform a detailed ablation study using the results from the Deep Search benchmark, as presented in Table \ref{tab:main_results_deepsearch_domain}. The study reveals a distinct and complementary duality in their roles, which stems from how they shape the policy during training. The \textit{Future Clarity Bonus} acts as a powerful \textit{exploitation} signal during training. By reinforcing known, high-quality decision sequences within the training data, it helps the model master the in-domain distribution, leading to a strong performance gain of \textit{+2.6} points on ID tasks.
Conversely, the \textit{Self-Calibrating Gradient Scaling} serves as a powerful \textit{regularization} mechanism during training, teaching the model how to behave when it is uncertain. By attenuating updates for high-entropy steps, it produces a final policy that is inherently more robust and less brittle. This learned robustness is then observed during testing on out-of-domain tasks, where the model faces novel inputs that induce high uncertainty. Because the policy has learned not to overreact in such situations, it exhibits superior generalization, providing a robust gain of \textit{+3.9} points on OOD tasks. This demonstrates that EMPG is not merely overfitting; instead, by learning a fundamental skill of how to handle uncertainty, it acquires a more resilient problem-solving approach that generalizes effectively. Crucially, the full EMPG model, which integrates both mechanisms, demonstrates a powerful synergy: the model learns to efficiently exploit known patterns while being robust to novel ones.


\paragraph{Enhancing Training Stability.}
Beyond improving sample efficiency, EMPG also significantly enhances the stability and robustness of the training process. A common failure mode in online RL fine-tuning is "policy collapse," where the agent's policy diverges late in training, leading to a catastrophic drop in performance. We visualize this phenomenon by tracking the KL Loss during training, as shown in Figure~\ref{fig:kl_loss_stability}. The DAPO baseline agent initially learns effectively, but its KL Loss becomes highly erratic after approximately 240 training steps, indicating severe instability. In contrast, the EMPG-enhanced agent maintains a low and stable KL Loss throughout the entire training run. This demonstrates that EMPG's mechanisms, particularly the self-calibrating gradient scaling, effectively regularize the policy updates, preventing the overly aggressive changes that can lead to divergence and ensuring a more reliable convergence to a high-performance policy. To ensure a fair comparison, \textbf{we select the checkpoint at 220 steps for both the baseline and EMPG} for final evaluation. Despite this, our method could continue to improve its performance with further training.

\paragraph{Step-Level vs. Token-Level Entropy Dynamics.}
Our work diverges from prior analyses \citep{wang2025beyond} by focusing on entropy at the "reason-act" step level rather than the token level. To validate this choice, we investigate whether the token-level observation—that RL updates primarily affect high-entropy tokens—holds at the step level. We analyze over 9,000 steps on ALFWorld and plot the average entropy change for steps, binned by their initial entropy percentile (Figure \ref{fig:step_entropy_change}). Our findings are significant: unlike at the token level, even steps with very low initial entropy (e.g., the 15\%-20\% percentile) still undergo substantial average entropy changes. This shows the dynamics do not transfer; a confident step can still require significant policy updates. This key finding underscores the importance of our step-centric approach and motivates the design of EMPG to modulate updates across the entire confidence spectrum.

\paragraph{Analysis of Learning Dynamics.}
An analysis of the learning dynamics, presented in Figure \ref{fig:learning_curves}, clearly reveals EMPG's critical role in overcoming the performance limitations of baseline methods. Across all experiments on both the ALFWorld and WebShop benchmarks, the baseline agents consistently reach a distinct performance plateau, where their learning stagnates and the success rate ceases to improve. In stark contrast, the EMPG-enhanced agents decisively break through this performance ceiling. By providing a richer and more effective learning signal, EMPG enables the agents to sustain their learning momentum, pushing beyond the baseline's peak and ultimately converging to a significantly higher final success rate. This demonstrates that EMPG is not just accelerating learning, but is fundamentally guiding the agent to discover superior policies that are otherwise inaccessible, effectively escaping the local optima where the baseline methods become trapped.

\section{Conclusion}

In this work, we introduced Entropy-Modulated Policy Gradients (EMPG), a novel and principled framework to alleviate the long-standing credit assignment problem in long-horizon LLM agent training. By leveraging the intrinsic uncertainty of the agent's "reasoning-action" steps, EMPG dynamically re-calibrates the policy gradient, moving beyond the limitations of sparse, end-of-task rewards. Our method directly addresses the dual challenges of standard policy gradients: it amplifies updates for confident and correct actions, strongly penalizes confident but incorrect steps, and attenuates updates for uncertain steps to promote stability. Through comprehensive experiments on challenging benchmarks, including WebShop, ALFWorld, and Deep Search, we demonstrated substantial performance gains over strong baselines like GRPO and DAPO. More fundamentally, our work addresses a general challenge inherent in policy gradient methods operating in high-dimensional action spaces: the "entropy-gradient coupling" problem. We frame EMPG not as a domain-specific technique but as a general-purpose method for variance reduction and credit assignment, using the policy's own uncertainty as an adaptive, step-level baseline.

Our findings suggest that an agent's intrinsic uncertainty is a powerful, yet underexplored, signal for self-supervision in complex decision-making processes. EMPG provides a scalable alternative to costly process-based reward models, forging a dense, informative learning signal from minimal external feedback. For future work, we plan to explore the application of EMPG to other long-horizon tasks, such as embodied AI and multi-agent collaboration. We believe that this work lays a foundational stone for developing more efficient, robust, and self-correcting autonomous agents.


\clearpage

\bibliographystyle{plainnat}
\bibliography{main}

\begin{thebibliography}{49}
\providecommand{\natexlab}[1]{#1}
\providecommand{\url}[1]{\texttt{#1}}
\expandafter\ifx\csname urlstyle\endcsname\relax
  \providecommand{\doi}[1]{doi: #1}\else
  \providecommand{\doi}{doi: \begingroup \urlstyle{rm}\Url}\fi

\bibitem[Agarwal et~al.(2025)Agarwal, Zhang, Yuan, Han, and Peng]{agarwal2025unreasonable}
Shivam Agarwal, Zimin Zhang, Lifan Yuan, Jiawei Han, and Hao Peng.
\newblock The unreasonable effectiveness of entropy minimization in llm reasoning.
\newblock \emph{arXiv preprint arXiv:2505.15134}, 2025.

\bibitem[Alzubi et~al.(2025)Alzubi, Brooks, Chiniya, Contente, von Gerlach, Irwin, Jiang, Kaz, Nguyen, Oh, et~al.]{alzubi2025open}
Salaheddin Alzubi, Creston Brooks, Purva Chiniya, Edoardo Contente, Chiara von Gerlach, Lucas Irwin, Yihan Jiang, Arda Kaz, Windsor Nguyen, Sewoong Oh, et~al.
\newblock Open deep search: Democratizing search with open-source reasoning agents.
\newblock \emph{arXiv preprint arXiv:2503.20201}, 2025.

\bibitem[Bellemare et~al.(2016)Bellemare, Srinivasan, Ostrovski, Schaul, Saxton, and Munos]{bellemare2016unifying}
Marc Bellemare, Sriram Srinivasan, Georg Ostrovski, Tom Schaul, David Saxton, and Remi Munos.
\newblock Unifying count-based exploration and intrinsic motivation.
\newblock \emph{Advances in neural information processing systems}, 29, 2016.

\bibitem[Chen et~al.(2025)Chen, Chen, Wang, and Yang]{chen2025seed}
Minghan Chen, Guikun Chen, Wenguan Wang, and Yi~Yang.
\newblock Seed-grpo: Semantic entropy enhanced grpo for uncertainty-aware policy optimization.
\newblock \emph{arXiv preprint arXiv:2505.12346}, 2025.

\bibitem[Cheng et~al.(2025)Cheng, Huang, Zhu, Dai, Zhao, Zhang, and Wei]{cheng2025reasoning}
Daixuan Cheng, Shaohan Huang, Xuekai Zhu, Bo~Dai, Wayne~Xin Zhao, Zhenliang Zhang, and Furu Wei.
\newblock Reasoning with exploration: An entropy perspective.
\newblock \emph{arXiv preprint arXiv:2506.14758}, 2025.

\bibitem[Deng et~al.(2023)Deng, Gu, Zheng, Chen, Stevens, Wang, Sun, and Su]{deng2023mind2web}
Xiang Deng, Yu~Gu, Boyuan Zheng, Shijie Chen, Sam Stevens, Boshi Wang, Huan Sun, and Yu~Su.
\newblock Mind2web: Towards a generalist agent for the web.
\newblock In \emph{Advances in Neural Information Processing Systems}, 2023.

\bibitem[Deng et~al.(2024)Deng, Dou, Zhu, Wen, Xiong, Wang, and Chen]{deng2024novice}
Zhirui Deng, Zhicheng Dou, Yutao Zhu, Ji-Rong Wen, Ruibin Xiong, Mang Wang, and Weipeng Chen.
\newblock From novice to expert: Llm agent policy optimization via step-wise reinforcement learning.
\newblock \emph{arXiv preprint arXiv:2411.03817}, 2024.

\bibitem[Feng et~al.(2025)Feng, Xue, Liu, and An]{feng2025group}
Lang Feng, Zhenghai Xue, Tingcong Liu, and Bo~An.
\newblock Group-in-group policy optimization for llm agent training.
\newblock \emph{arXiv preprint arXiv:2505.10978}, 2025.

\bibitem[Gao et~al.(2025)Gao, Chen, Zhou, and Dai]{gao2025one}
Zitian Gao, Lynx Chen, Joey Zhou, and Bryan Dai.
\newblock One-shot entropy minimization.
\newblock \emph{arXiv preprint arXiv:2505.20282}, 2025.

\bibitem[He et~al.(2025)He, Huang, Feng, Lin, Zhang, Li, et~al.]{he2025pasa}
Yichen He, Guanhua Huang, Peiyuan Feng, Yuan Lin, Yuchen Zhang, Hang Li, et~al.
\newblock Pasa: An llm agent for comprehensive academic paper search.
\newblock \emph{arXiv preprint arXiv:2501.10120}, 2025.

\bibitem[Ho et~al.(2020)Ho, Nguyen, Sugawara, and Aizawa]{ho2020constructing}
Xanh Ho, Anh-Khoa~Duong Nguyen, Saku Sugawara, and Akiko Aizawa.
\newblock Constructing a multi-hop qa dataset for comprehensive evaluation of reasoning steps.
\newblock \emph{arXiv preprint arXiv:2011.01060}, 2020.

\bibitem[Jimenez et~al.(2023)Jimenez, Yang, Wettig, Yao, Pei, Press, and Narasimhan]{jimenez2023swe}
Carlos~E Jimenez, John Yang, Alexander Wettig, Shunyu Yao, Kexin Pei, Ofir Press, and Karthik Narasimhan.
\newblock Swe-bench: Can language models resolve real-world github issues?
\newblock \emph{arXiv preprint arXiv:2310.06770}, 2023.

\bibitem[Jin et~al.(2025)Jin, Zeng, Yue, Yoon, Arik, Wang, Zamani, and Han]{jin2025search}
Bowen Jin, Hansi Zeng, Zhenrui Yue, Jinsung Yoon, Sercan Arik, Dong Wang, Hamed Zamani, and Jiawei Han.
\newblock Search-r1: Training llms to reason and leverage search engines with reinforcement learning.
\newblock \emph{arXiv preprint arXiv:2503.09516}, 2025.

\bibitem[Joshi et~al.(2017)Joshi, Choi, Weld, and Zettlemoyer]{joshi2017triviaqa}
Mandar Joshi, Eunsol Choi, Daniel~S Weld, and Luke Zettlemoyer.
\newblock Triviaqa: A large scale distantly supervised challenge dataset for reading comprehension.
\newblock \emph{arXiv preprint arXiv:1705.03551}, 2017.

\bibitem[Kazemnejad et~al.(2024)Kazemnejad, Aghajohari, Portelance, Sordoni, Reddy, Courville, and Le~Roux]{kazemnejad2024vineppo}
Amirhossein Kazemnejad, Milad Aghajohari, Eva Portelance, Alessandro Sordoni, Siva Reddy, Aaron Courville, and Nicolas Le~Roux.
\newblock Vineppo: Accurate credit assignment in rl for llm mathematical reasoning.
\newblock In \emph{The 4th Workshop on Mathematical Reasoning and AI at NeurIPS'24}, 2024.

\bibitem[Klyubin et~al.(2005)Klyubin, Polani, and Nehaniv]{klyubin2005empowerment}
Alexander~S Klyubin, Daniel Polani, and Chrystopher~L Nehaniv.
\newblock Empowerment: A universal agent-centric measure of control.
\newblock In \emph{IEEE Congress on Evolutionary Computation}, 2005.

\bibitem[Kwiatkowski et~al.(2019)Kwiatkowski, Palomaki, Redfield, Collins, Parikh, Alberti, Epstein, Polosukhin, Devlin, Lee, et~al.]{kwiatkowski2019natural}
Tom Kwiatkowski, Jennimaria Palomaki, Olivia Redfield, Michael Collins, Ankur Parikh, Chris Alberti, Danielle Epstein, Illia Polosukhin, Jacob Devlin, Kenton Lee, et~al.
\newblock Natural questions: a benchmark for question answering research.
\newblock \emph{Transactions of the Association for Computational Linguistics}, 2019.

\bibitem[Li et~al.(2025)Li, Dong, Jin, Zhang, Zhou, Zhu, Zhang, and Dou]{li2025search}
Xiaoxi Li, Guanting Dong, Jiajie Jin, Yuyao Zhang, Yujia Zhou, Yutao Zhu, Peitian Zhang, and Zhicheng Dou.
\newblock Search-o1: Agentic search-enhanced large reasoning models.
\newblock \emph{arXiv preprint arXiv:2501.05366}, 2025.

\bibitem[Li(2025)]{li2025logit}
Yingru Li.
\newblock Logit dynamics in softmax policy gradient methods.
\newblock \emph{arXiv preprint arXiv:2506.12912}, 2025.

\bibitem[Lightman et~al.(2023)Lightman, Kosaraju, Burda, Edwards, Baker, Lee, Leike, Schulman, Sutskever, and Cobbe]{lightman2023let}
Hunter Lightman, Vineet Kosaraju, Yuri Burda, Harrison Edwards, Bowen Baker, Teddy Lee, Jan Leike, John Schulman, Ilya Sutskever, and Karl Cobbe.
\newblock Let's verify step by step.
\newblock In \emph{International Conference on Learning Representations}, 2023.

\bibitem[Liu et~al.(2024)Liu, Wang, Liu, Zeng, Yan, Sun, Liu, and Zhou]{liu2024improving}
Jiacai Liu, Chaojie Wang, Chris~Yuhao Liu, Liang Zeng, Rui Yan, Yiwen Sun, Yang Liu, and Yahui Zhou.
\newblock Improving multi-step reasoning abilities of large language models with direct advantage policy optimization.
\newblock \emph{arXiv preprint arXiv:2412.18279}, 2024.

\bibitem[Ng et~al.(1999)Ng, Harada, and Russell]{ng1999policy}
Andrew~Y Ng, Daishi Harada, and Stuart Russell.
\newblock Policy invariance under reward transformations: Theory and application to reward shaping.
\newblock In \emph{International Conference on Machine Learning}, 1999.

\bibitem[Pathak et~al.(2017)Pathak, Agrawal, Efros, and Darrell]{pathak2017curiosity}
Deepak Pathak, Pulkit Agrawal, Alexei~A Efros, and Trevor Darrell.
\newblock Curiosity-driven exploration by self-supervised prediction.
\newblock In \emph{International Conference on Machine Learning}, 2017.

\bibitem[R{\'e}nyi(1961)]{renyi1961measures}
Alfr{\'e}d R{\'e}nyi.
\newblock On measures of entropy and information.
\newblock In \emph{Proceedings of the fourth Berkeley symposium on mathematical statistics and probability, volume 1: contributions to the theory of statistics}, 1961.

\bibitem[Schulman et~al.(2017)Schulman, Wolski, Dhariwal, Radford, and Klimov]{schulman2017proximal}
John Schulman, Filip Wolski, Prafulla Dhariwal, Alec Radford, and Oleg Klimov.
\newblock Proximal policy optimization algorithms.
\newblock \emph{arXiv preprint arXiv:1707.06347}, 2017.

\bibitem[Seed et~al.(2025)Seed, Chen, Fan, Liu, Liu, Lin, Wang, Wang, Wei, Xu, et~al.]{seed2025seed1}
ByteDance Seed, Jiaze Chen, Tiantian Fan, Xin Liu, Lingjun Liu, Zhiqi Lin, Mingxuan Wang, Chengyi Wang, Xiangpeng Wei, Wenyuan Xu, et~al.
\newblock Seed1. 5-thinking: Advancing superb reasoning models with reinforcement learning.
\newblock \emph{arXiv preprint arXiv:2504.13914}, 2025.

\bibitem[Shao et~al.(2024)Shao, Wang, Zhu, Xu, Song, Bi, Zhang, Zhang, Li, Wu, et~al.]{shao2024deepseekmath}
Zhihong Shao, Peiyi Wang, Qihao Zhu, Runxin Xu, Junxiao Song, Xiao Bi, Haowei Zhang, Mingchuan Zhang, YK~Li, Yang Wu, et~al.
\newblock Deepseekmath: Pushing the limits of mathematical reasoning in open language models.
\newblock \emph{arXiv preprint arXiv:2402.03300}, 2024.

\bibitem[Sheng et~al.(2024)Sheng, Zhang, Ye, Wu, Zhang, Zhang, Peng, Lin, and Wu]{sheng2024hybridflow}
Guangming Sheng, Chi Zhang, Zilingfeng Ye, Xibin Wu, Wang Zhang, Ru~Zhang, Yanghua Peng, Haibin Lin, and Chuan Wu.
\newblock Hybridflow: A flexible and efficient rlhf framework.
\newblock \emph{arXiv preprint arXiv: 2409.19256}, 2024.

\bibitem[Shridhar et~al.(2021)Shridhar, Yuan, Cote, Bisk, Trischler, and Hausknecht]{shridhar2021alfworld}
Mohit Shridhar, Xingdi Yuan, Marc-Alexandre Cote, Yonatan Bisk, Adam Trischler, and Matthew Hausknecht.
\newblock Alfworld: Aligning text and embodied environments for interactive learning.
\newblock In \emph{International Conference on Learning Representations}, 2021.

\bibitem[Vanlioglu(2025)]{vanlioglu2025entropy}
Abdullah Vanlioglu.
\newblock Entropy-guided sequence weighting for efficient exploration in rl-based llm fine-tuning.
\newblock \emph{arXiv preprint arXiv:2503.22456}, 2025.

\bibitem[Wang et~al.(2025)Wang, Yu, Gao, Zheng, Liu, Lu, Dang, Chen, Yang, Zhang, Liu, Yang, Zhao, Yue, Song, Yu, Huang, and Lin]{wang2025beyond}
Shenzhi Wang, Le~Yu, Chang Gao, Chujie Zheng, Shixuan Liu, Rui Lu, Kai Dang, Xionghui Chen, Jianxin Yang, Zhenru Zhang, Yuqiong Liu, An~Yang, Andrew Zhao, Yang Yue, Shiji Song, Bowen Yu, Gao Huang, and Junyang Lin.
\newblock Beyond the 80/20 rule: High-entropy minority tokens drive effective reinforcement learning for llm reasoning.
\newblock \emph{arXiv preprint arXiv:2506.01939}, 2025.

\bibitem[Wei et~al.(2022)Wei, Wang, Schuurmans, Bosma, Xia, Chi, Le, Zhou, et~al.]{wei2022chain}
Jason Wei, Xuezhi Wang, Dale Schuurmans, Maarten Bosma, Fei Xia, Ed~Chi, Quoc~V Le, Denny Zhou, et~al.
\newblock Chain-of-thought prompting elicits reasoning in large language models.
\newblock In \emph{Advances in neural information processing systems}, 2022.

\bibitem[Wei et~al.(2025{\natexlab{a}})Wei, Duchenne, Copet, Carbonneaux, Zhang, Fried, Synnaeve, Singh, and Wang]{wei2025swe}
Yuxiang Wei, Olivier Duchenne, Jade Copet, Quentin Carbonneaux, Lingming Zhang, Daniel Fried, Gabriel Synnaeve, Rishabh Singh, and Sida~I Wang.
\newblock Swe-rl: Advancing llm reasoning via reinforcement learning on open software evolution.
\newblock \emph{arXiv preprint arXiv:2502.18449}, 2025{\natexlab{a}}.

\bibitem[Wei et~al.(2025{\natexlab{b}})Wei, Yao, Liu, Zhang, Lu, Qiu, Yu, Xu, Zhang, Yin, et~al.]{wei2025webagent}
Zhepei Wei, Wenlin Yao, Yao Liu, Weizhi Zhang, Qin Lu, Liang Qiu, Changlong Yu, Puyang Xu, Chao Zhang, Bing Yin, et~al.
\newblock Webagent-r1: Training web agents via end-to-end multi-turn reinforcement learning.
\newblock \emph{arXiv preprint arXiv:2505.16421}, 2025{\natexlab{b}}.

\bibitem[Wu et~al.(2025)Wu, Yin, Jiang, Wang, Xi, Fang, Zhang, He, Zhou, Xie, et~al.]{wu2025webwalker}
Jialong Wu, Wenbiao Yin, Yong Jiang, Zhenglin Wang, Zekun Xi, Runnan Fang, Linhai Zhang, Yulan He, Deyu Zhou, Pengjun Xie, et~al.
\newblock Webwalker: Benchmarking llms in web traversal.
\newblock \emph{arXiv preprint arXiv:2501.07572}, 2025.

\bibitem[Yan et~al.(2023)Yan, Yang, Zhu, Lin, Li, Wang, Yang, Zhong, McAuley, Gao, et~al.]{yan2023gpt}
An~Yan, Zhengyuan Yang, Wanrong Zhu, Kevin Lin, Linjie Li, Jianfeng Wang, Jianwei Yang, Yiwu Zhong, Julian McAuley, Jianfeng Gao, et~al.
\newblock Gpt-4v in wonderland: Large multimodal models for zero-shot smartphone gui navigation.
\newblock \emph{arXiv preprint arXiv:2311.07562}, 2023.

\bibitem[Yang et~al.(2018)Yang, Qi, Zhang, Bengio, Cohen, Salakhutdinov, and Manning]{yang2018hotpotqa}
Zhilin Yang, Peng Qi, Saizheng Zhang, Yoshua Bengio, William~W Cohen, Ruslan Salakhutdinov, and Christopher~D Manning.
\newblock Hotpotqa: A dataset for diverse, explainable multi-hop question answering.
\newblock \emph{arXiv preprint arXiv:1809.09600}, 2018.

\bibitem[Yao et~al.(2022)Yao, Chen, Yang, and Narasimhan]{yao2022webshop}
Shunyu Yao, Howard Chen, John Yang, and Karthik Narasimhan.
\newblock Webshop: Towards scalable real-world web interaction with grounded language agents.
\newblock In \emph{Advances in Neural Information Processing Systems}, 2022.

\bibitem[Yao et~al.(2023)Yao, Zhao, Yu, Du, Shafran, Narasimhan, and Cao]{yao2023react}
Shunyu Yao, Jeffrey Zhao, Dian Yu, Nan Du, Izhak Shafran, Karthik Narasimhan, and Yuan Cao.
\newblock React: Synergizing reasoning and acting in language models.
\newblock In \emph{International Conference on Learning Representations}, 2023.

\bibitem[Yu et~al.(2025)Yu, Zhang, Zhu, Yuan, Zuo, Yue, Fan, Liu, Liu, Liu, Lin, Lin, Ma, Sheng, Tong, Zhang, Zhang, Zhang, Zhu, Zhu, Chen, Chen, Wang, Yu, Dai, Song, Wei, Zhou, Liu, Ma, Zhang, Yan, Qiao, Wu, and Wang]{yu2025dapo}
Qiying Yu, Zheng Zhang, Ruofei Zhu, Yufeng Yuan, Xiaochen Zuo, Yu~Yue, Tiantian Fan, Gaohong Liu, Lingjun Liu, Xin Liu, Haibin Lin, Zhiqi Lin, Bole Ma, Guangming Sheng, Yuxuan Tong, Chi Zhang, Mofan Zhang, Wang Zhang, Hang Zhu, Jinhua Zhu, Jiaze Chen, Jiangjie Chen, Chengyi Wang, Honglin Yu, Weinan Dai, Yuxuan Song, Xiang Wei, Haodong Zhou, Jingjing Liu, Wei Ma, Ya-Qin Zhang, Lin Yan, Mu~Qiao, Yong-Xu Wu, and Mingxuan Wang.
\newblock Dapo: An open-source llm reinforcement learning system at scale.
\newblock \emph{arXiv preprint arXiv:2503.14476}, 2025.

\bibitem[Zhang et~al.(2024)Zhang, Li, Li, Shi, and Jin]{zhang2024codeagent}
Kechi Zhang, Jia Li, Ge~Li, Xianjie Shi, and Zhi Jin.
\newblock Codeagent: Enhancing code generation with tool-integrated agent systems for real-world repo-level coding challenges.
\newblock \emph{arXiv preprint arXiv:2401.07339}, 2024.

\bibitem[Zhang et~al.(2025{\natexlab{a}})Zhang, Yao, Liu, Wang, Lai, Ye, Song, and Tao]{zhang2025consistent}
Kongcheng Zhang, Qi~Yao, Shunyu Liu, Yingjie Wang, Baisheng Lai, Jieping Ye, Mingli Song, and Dacheng Tao.
\newblock Consistent paths lead to truth: Self-rewarding reinforcement learning for llm reasoning.
\newblock \emph{arXiv preprint arXiv:2506.08745}, 2025{\natexlab{a}}.

\bibitem[Zhang et~al.(2025{\natexlab{b}})Zhang, Wu, Zhang, Zhao, and Bian]{zhang2025right}
Qingyang Zhang, Haitao Wu, Changqing Zhang, Peilin Zhao, and Yatao Bian.
\newblock Right question is already half the answer: Fully unsupervised llm reasoning incentivization.
\newblock \emph{arXiv preprint arXiv:2504.05812}, 2025{\natexlab{b}}.

\bibitem[Zhang et~al.(2025{\natexlab{c}})Zhang, Wen, Wu, and Huang]{zhang2025edge}
Xingjian Zhang, Siwei Wen, Wenjun Wu, and Lei Huang.
\newblock Edge-grpo: Entropy-driven grpo with guided error correction for advantage diversity.
\newblock \emph{arXiv preprint arXiv:2507.21848}, 2025{\natexlab{c}}.

\bibitem[Zhang et~al.(2025{\natexlab{d}})Zhang, Li, Cui, Cai, Liu, Fu, Huang, Zhao, Zhang, Chen, et~al.]{zhang2025siren}
Yue Zhang, Yafu Li, Leyang Cui, Deng Cai, Lemao Liu, Tingchen Fu, Xinting Huang, Enbo Zhao, Yu~Zhang, Yulong Chen, et~al.
\newblock Siren’s song in the ai ocean: A survey on hallucination in large language models.
\newblock \emph{Computational Linguistics}, pages 1--45, 2025{\natexlab{d}}.

\bibitem[Zhao et~al.(2025)Zhao, Kang, Feng, Levine, and Song]{zhao2025learning}
Xuandong Zhao, Zhewei Kang, Aosong Feng, Sergey Levine, and Dawn Song.
\newblock Learning to reason without external rewards.
\newblock \emph{arXiv preprint arXiv:2505.19590}, 2025.

\bibitem[Zheng et~al.(2025)Zheng, Fu, Hu, Cai, Ye, Lu, and Liu]{zheng2025deepresearcher}
Yuxiang Zheng, Dayuan Fu, Xiangkun Hu, Xiaojie Cai, Lyumanshan Ye, Pengrui Lu, and Pengfei Liu.
\newblock Deepresearcher: Scaling deep research via reinforcement learning in real-world environments.
\newblock \emph{arXiv preprint arXiv:2504.03160}, 2025.

\bibitem[Ziebart et~al.(2008)Ziebart, Maas, Bagnell, Dey, et~al.]{ziebart2008maximum}
Brian~D Ziebart, Andrew~L Maas, J~Andrew Bagnell, Anind~K Dey, et~al.
\newblock Maximum entropy inverse reinforcement learning.
\newblock In \emph{The Association for the Advancement of Artificial Intelligence}, 2008.

\bibitem[Zuo et~al.(2025)Zuo, Zhang, Sheng, Qu, Cui, Zhu, Li, Zhang, Long, Hua, et~al.]{zuo2025ttrl}
Yuxin Zuo, Kaiyan Zhang, Li~Sheng, Shang Qu, Ganqu Cui, Xuekai Zhu, Haozhan Li, Yuchen Zhang, Xinwei Long, Ermo Hua, et~al.
\newblock Ttrl: Test-time reinforcement learning.
\newblock \emph{arXiv preprint arXiv:2504.16084}, 2025.

\end{thebibliography}

\clearpage

\beginappendix

\appendix
\section{Proof of Proposition \ref{prop:expected_grad_norm}}
\label{app:proof}
We aim to prove that $\mathbb{E}_{a_k \sim \pi} \left[ ||\nabla_{z} \log \pi_k||^2 \right] = 1 - \sum_{j=1}^{|V|} \pi_j^2$. The proof requires the result for the gradient norm of a single action $a_k$, which we state as a lemma.

\newtheorem*{lemma_}{Lemma}
\begin{lemma_}
The squared L2-norm of the score function with respect to the logits, for a chosen action $a_k$, is given by: $||\nabla_{z} \log \pi_k||^2 = 1 - 2\pi_k + \sum_{j=1}^{|V|} \pi_j^2$.
\end{lemma_}
\begin{proof}[Proof of Lemma]
Let the logits be $z = (z_1, \dots, z_{|V|})$. The policy is $\pi_k = \exp(z_k) / \sum_{j} \exp(z_j)$. The partial derivative of the log-probability $\log \pi_k$ with respect to an arbitrary logit $z_i$ is $\frac{\partial \log \pi_k}{\partial z_i} = \delta_{ik} - \pi_i$, where $\delta_{ik}$ is the Kronecker delta. The squared L2-norm of the gradient vector $\nabla_z \log \pi_k$ is therefore:
\begin{align*}
    ||\nabla_z \log \pi_k||^2 &= \sum_{i=1}^{|V|} ( \delta_{ik} - \pi_i )^2 = (1-\pi_k)^2 + \sum_{i \neq k} (-\pi_i)^2 \\
    &= (1 - 2\pi_k + \pi_k^2) + \sum_{i \neq k} \pi_i^2 = 1 - 2\pi_k + \sum_{j=1}^{|V|} \pi_j^2
\end{align*}
\end{proof}

\begin{proof}[Proof of Proposition \ref{prop:expected_grad_norm}]
The expectation is taken over all possible choices of action $a_k$ according to the policy distribution $\pi$. Using the result from the lemma:
\begin{align*}
    \mathbb{E}_{k \sim \pi} \left[ ||\nabla_{z} \log \pi_k||^2 \right] &= \sum_{k=1}^{|V|} \pi_k \cdot \left( ||\nabla_{z} \log \pi_k||^2 \right) \\
    &= \sum_{k=1}^{|V|} \pi_k \left( 1 - 2\pi_k + \sum_{j=1}^{|V|} \pi_j^2 \right) \\
    &= \sum_{k=1}^{|V|} \pi_k - 2\sum_{k=1}^{|V|} \pi_k^2 + \sum_{k=1}^{|V|} \pi_k \left( \sum_{j=1}^{|V|} \pi_j^2 \right) \\
    &= 1 - 2\sum_{k=1}^{|V|} \pi_k^2 + \left( \sum_{j=1}^{|V|} \pi_j^2 \right) \left( \sum_{k=1}^{|V|} \pi_k \right) \quad \text{(Factor out constant term)} \\
    &= 1 - 2\sum_{k=1}^{|V|} \pi_k^2 + \left( \sum_{j=1}^{|V|} \pi_j^2 \right) \cdot 1 \\
    &= 1 - \sum_{k=1}^{|V|} \pi_k^2
\end{align*}
Recalling the definition of Rényi entropy of order 2, $H_2(\pi) = -\log(\sum_{j=1}^{|V|} \pi_j^2)$, we can identify the term $\sum \pi_j^2$ as the collision probability, which is equivalent to $\exp(-H_2(\pi))$. Substituting this into our result yields the final information-theoretic form:
\begin{equation*}
    \mathbb{E}_{k \sim \pi} \left[ ||\nabla_{z} \log \pi_k||^2 \right] = 1 - \exp(-H_2(\pi))
\end{equation*}
This completes the proof of the proposition.
\end{proof}

\clearpage

\section{Theoretical Foundation of the EMPG Update Rule}
\label{app:theoretical_derivation_empg}

In this section, we provide a rigorous theoretical justification for the Entropy-Modulated Policy Gradients (EMPG) algorithm. We demonstrate that the EMPG update rule can be formally derived as the gradient of a composite objective function, $J_{\text{EMPG}}(\theta)$. This interpretation substantiates that EMPG is a principled optimization method that reshapes the standard reinforcement learning objective to favor policies that are both effective and robust.

\subsection{The Standard Policy Gradient Objective}
We begin with the standard objective in policy-based reinforcement learning, which is to maximize the expected total return. In the context of sparse, outcome-based rewards, this objective simplifies to maximizing the expected advantage (return) of a trajectory $\tau$:
\begin{equation}
    J(\theta) = \E_{\tau \sim \pi_\theta}[A^{(\tau)}]
\end{equation}
where $A^{(\tau)}$ is the scalar return for a trajectory $\tau$ sampled from the policy $\pi_\theta$. The gradient of this objective is given by the Policy Gradient Theorem:
\begin{equation}
    \grad J(\theta) = \E_{\tau \sim \pi_\theta} \left[ \left( \sum_{t=0}^{T-1} \grad \logpi \right) A^{(\tau)} \right]
    \label{eq:standard_pg}
\end{equation}
For any single trajectory $\tau$, the gradient estimator is $\mathcal{G}^{(\tau)}(\theta) = A^{(\tau)} \sum_{t=0}^{T-1} \grad \logpi$. This formulation reveals the core issue identified in Proposition~\ref{prop:expected_grad_norm}: the contribution of each step's score function, $\grad \logpi$, is weighted uniformly by the trajectory's outcome $A^{(\tau)}$, while its norm is intrinsically coupled with the policy entropy $H_t$.

\subsection{The EMPG Composite Objective Function}
We posit that EMPG performs gradient ascent on a composite objective function $J_{\text{EMPG}}(\theta)$. This objective augments the standard RL objective with a term that explicitly accounts for policy uncertainty, thereby decoupling the learning signal's magnitude and direction from the policy's raw confidence. We define this objective as:
\begin{equation}
    J_{\text{EMPG}}(\theta) = J_{\text{extrinsic}}(\theta) + J_{\text{intrinsic}}(\theta)
\end{equation}
Here, $J_{\text{extrinsic}}(\theta)$ is a re-weighted extrinsic objective that addresses the gradient \textit{magnitude} problem, and $J_{\text{intrinsic}}(\theta)$ is an intrinsic objective that guides the policy's \textit{direction} towards states of higher certainty.

\subsubsection{The Re-weighted Extrinsic Objective}
The self-calibrating gradient scaling component of EMPG, $A^{(\tau)} \cdot \gH$, can be interpreted as performing an update on a modified extrinsic objective. Formally, we define a state-dependent weighting function $\omega(s_t, \theta) = \gH$, which is a function of the policy's entropy at state $s_t$. The gradient update for this component is:
\begin{equation}
    \mathcal{G}_{\text{extrinsic}}^{(\tau)}(\theta) = \sum_{t=0}^{T-1} A^{(\tau)} \cdot \omega(s_t^{(\tau)}, \theta) \cdot \grad \logpi
\end{equation}
This formulation is equivalent to optimizing the standard objective $J(\theta)$ under a \textit{state-dependent measure}, where the contribution of each state is re-weighted. While deriving a closed-form objective $J_{\text{extrinsic}}(\theta)$ is non-trivial because $\omega$ depends on $\theta$ in a complex manner (via batch statistics), this interpretation is sufficient to justify the update rule. The weighting function $\omega(s_t, \theta)$ serves as an adaptive, information-theoretic learning rate that directly counteracts the dynamics described in Proposition~\ref{prop:expected_grad_norm}. It amplifies the learning signal for confident (low-entropy) steps and dampens it for uncertain (high-entropy) steps, thus achieving a direct re-calibration of the gradient's magnitude.

\subsubsection{The Intrinsic Clarity Objective}
The Future Clarity Bonus can be modeled as the gradient of a well-defined intrinsic objective function. We define an intrinsic reward, $r^{\text{int}}_t$, awarded at step $t$ for transitioning to a state $s_{t+1}$ with high policy clarity:
\begin{definition*}[Clarity Reward]
The intrinsic clarity reward at step $t$ is a function of the policy entropy at the subsequent state $s_{t+1}$:
\begin{equation}
    r^{\text{int}}_t(s_{t+1}; \theta) = \zeta \cdot f(H(\pi_\theta(\cdot|s_{t+1}))) = \zeta \cdot \exp(-k' \cdot H_{\text{norm}, t+1})
\end{equation}
\end{definition*}
This reward incentivizes actions that lead to predictable future states. The corresponding intrinsic objective, $J_{\text{intrinsic}}(\theta)$, is the expected cumulative intrinsic reward:
\begin{equation}
    J_{\text{intrinsic}}(\theta) = \E_{\tau \sim \pi_\theta} \left[ \sum_{t=0}^{T-1} r^{\text{int}}_t(s_{t+1}; \theta) \right]
\end{equation}
Applying the policy gradient theorem to this objective, and using the immediate intrinsic reward as a one-step advantage estimate (a common form of advantage shaping), yields the gradient:
\begin{align}
    \grad J_{\text{intrinsic}}(\theta) &= \E_{\tau \sim \pi_\theta} \left[ \sum_{t=0}^{T-1} \left( \grad \logpi \right) r^{\text{int}}_t(s_{t+1}; \theta) \right] \\
    &= \E_{\tau_i \sim \pi_\theta} \left[ \sum_{t=0}^{T_i-1} \left( \grad \logpi \right) \zeta \cdot \fH \right]
\end{align}
This gradient precisely matches the Future Clarity Bonus component of the EMPG update.

\subsection{Synthesis: The Full EMPG Gradient}
By combining the gradients of the extrinsic and intrinsic objectives, we recover the full EMPG gradient estimator for a single trajectory $\tau$:
\begin{align}
    \mathcal{G}_{\text{EMPG}}^{(\tau)}(\theta) &= \mathcal{G}_{\text{extrinsic}}^{(\tau)}(\theta) + \grad J_{\text{intrinsic}}(\theta)|_{\tau} \\
    &= \sum_{t=0}^{T-1} A^{(\tau)} \cdot \gH \cdot \grad \logpi + \sum_{t=0}^{T-1} \zeta \cdot \fH \cdot \grad \logpi \\
    &= \sum_{t=0}^{T-1} \left( A^{(\tau)} \cdot \gH + \zeta \cdot \fH \right) \grad \logpi
\end{align}
This derivation confirms that the EMPG algorithm performs a principled gradient ascent on the composite objective $J_{\text{EMPG}}(\theta)$. This objective function holistically reshapes the optimization landscape by (1) adaptively scaling the extrinsic reward signal to ensure its magnitude is motivationally salient rather than merely a function of policy entropy, and (2) introducing an intrinsic drive towards robust, predictable solution paths. This dual-pronged approach provides a theoretical foundation for why EMPG successfully mitigates the challenges posed by the inherent dynamics of standard policy gradients.

\section{Experimental Settings}

This appendix provides a detailed description of the experimental settings, hardware configurations, and hyperparameter choices for our experiments across the three main benchmarks. Due to the differences in training frameworks and task environments, the settings for WebShop/ALFWorld and Deep Search are described in separate subsections.

\subsection{WebShop and ALFWorld Experiments}
Our experiments on WebShop and ALFWorld are conducted within the \href{https://github.com/langfengQ/verl-agent}{\textbf{Verl-Agent}} framework, an extension of the \textbf{veRL} \cite{sheng2024hybridflow} training codebase specifically designed for training large language model (LLM) agents via reinforcement learning. Verl-Agent provides a powerful and scalable platform for long-horizon, multi-turn RL training by enabling fully customizable per-step input structures, history management, and memory modules. It supports a diverse set of RL algorithms and a rich suite of agent environments, making it highly suitable for our work.

For a fair comparison, all experiments were re-executed on our hardware platform. While the original experiments were performed using H200 GPUs, our work utilized A100 GPUs due to resource constraints. We observed that the original training scripts for the Qwen2.5-1.5B-Instruct model, designed for 2 $\times$ H100, would result in out-of-memory errors on A100s. Therefore, we used 4 $\times$ A100 GPUs for the 1.5B models and 8 $\times$ A100 GPUs for the 7B models. All baselines were re-trained under the same hardware, seeds, and settings to ensure strict comparability. The key hyperparameters for these experiments are summarized in Table \ref{tab:webshop_alfworld_params}.

\begin{table}[h]
    \centering
    \caption{Key Hyperparameters for WebShop and ALFWorld Experiments.}
    \label{tab:webshop_alfworld_params}
    \begin{tabular}{ll}
        \toprule
        \textbf{Parameter} & \textbf{Value} \\
        \midrule
        Actor Learning Rate & 1e-6 \\
        KL Loss Coefficient & 0.01 \\
        KL Penalty & low var kl \\
        Entropy Coefficient & 0.001 \\
        Clip High (DAPO) & 0.28 \\
        Clip Low (DAPO) & 0.2 \\
        Clip Low/High (GRPO) & 0.2 \\
        Batch Size & 16 \\
        Training Step & 150 \\
        Rollout Group Size & 8 \\
        Rollout Temperature & 1.0 \\
        $\zeta$ & 0.05 \\
        $k$, $k'$ & 1.0 \\
        Max Actions (ALFWorld) & 50 \\
        Max Actions (WebShop) & 15 \\
        History Observation & 2 \\
        GPUs & 4 $\times$ A100 (1.5B), 8 $\times$ A100 (7B) \\
        \bottomrule
    \end{tabular}
\end{table}

\subsection{Deep Search Experiments}
Our experiments on the Deep Search task were conducted using a proprietary RL training framework from ByteDance. The agent was equipped with two primary tools: Bing Search as the search engine and a web viewer tool capable of reading web page content and summarizing long articles.

A key part of the Deep Search training was the data curation process. We constructed a unique training dataset of 17,000 instances by filtering from a variety of public benchmarks, including WebWalker \cite{wu2025webwalker}, HotpotQA \cite{yang2018hotpotqa}, 2WikiMultiHopQA \cite{ho2020constructing}, NaturalQuestions \cite{kwiatkowski2019natural}, and TriviaQA \cite{joshi2017triviaqa}. We gratefully acknowledge the initial data collection and preliminary filtering by the DeepResearcher team \cite{zheng2025deepresearcher}. We performed two deeper filtering steps:
\begin{enumerate}
    \item \textbf{Direct Answer Filtering:} We sampled 5 results per question using Doubao-Seed-1.6 (Thinking) \cite{seed2025seed1}. We then filtered out all questions that could be answered directly (where at least one of the 5 results was correct) to ensure the agent learns to use its search tools rather than relying on memorized answers.
    \item \textbf{Agent Workflow Filtering:} We further filtered the dataset by sampling 8 results using a search workflow built on Doubao-Seed-1.6 (Thinking). We removed data points that were "stably all-correct" to focus the RL training on more challenging instances and improve training efficiency.
\end{enumerate}
The key hyperparameters for the RL training on the Deep Search task are detailed in Table \ref{tab:deepsearch_params}.

\begin{table}[h]
    \centering
    \caption{Key Hyperparameters for Deep Search Experiments.}
    \label{tab:deepsearch_params}
    \begin{tabular}{ll}
        \toprule
        \textbf{Parameter} & \textbf{Value} \\
        \midrule
        Actor Learning Rate & 1e-6 \\
        KL Loss Coefficient & 0.001 \\
        KL Penalty & low var kl \\
        Entropy Coefficient & 0.0 \\
        Clip High & 0.28 \\
        Clip Low & 0.2 \\
        Batch Size & 64 \\
        Training Step & 220 \\
        Rollout Group Size & 16 \\
        Rollout Temperature & 1.0 \\
        $\zeta$ & 0.1 \\
        $k$, $k'$ & 1.0 \\
        Max Actions & 15 \\
        GPUs & 32 $\times$ A100 \\
        \bottomrule
    \end{tabular}
\end{table}

\section{Analysis of Learning Dynamics}
\counterwithin{figure}{section}
\setcounter{figure}{0}

This section provides a detailed visualization of the learning dynamics, complementing the analysis in the main body of the paper. Figure \ref{fig:learning_curves} illustrates the training progress of EMPG-enhanced agents compared to their baseline counterparts (GRPO and DAPO) on both the WebShop and ALFWorld benchmarks. As shown in the learning curves, the baseline agents consistently hit a performance ceiling, with their success rates stagnating early in the training process. In contrast, our EMPG-enhanced agents overcome this plateau, sustaining their learning momentum to achieve significantly higher final success rates across all settings. This evidence supports our central claim that EMPG provides a more effective learning signal, enabling agents to escape the local optima that trap standard policy gradient methods.

\begin{figure}[t]
    \centering
    
    \begin{subfigure}[b]{0.48\textwidth}
        \centering
        \includegraphics[width=\linewidth]{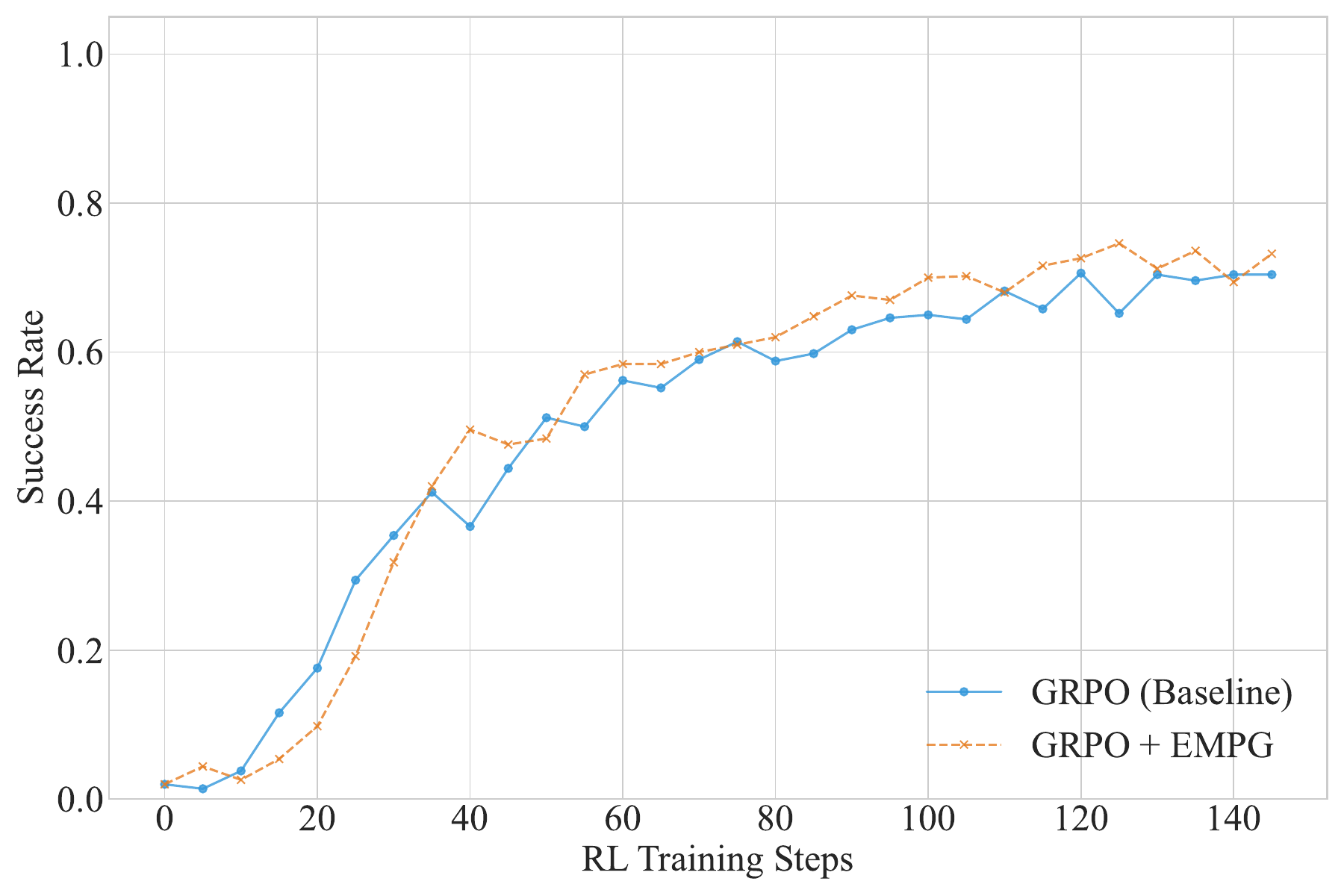}
        \caption{WebShop: GRPO}
        \label{fig:webshop_grpo}
    \end{subfigure}
    \hfill 
    \begin{subfigure}[b]{0.48\textwidth}
        \centering
        \includegraphics[width=\linewidth]{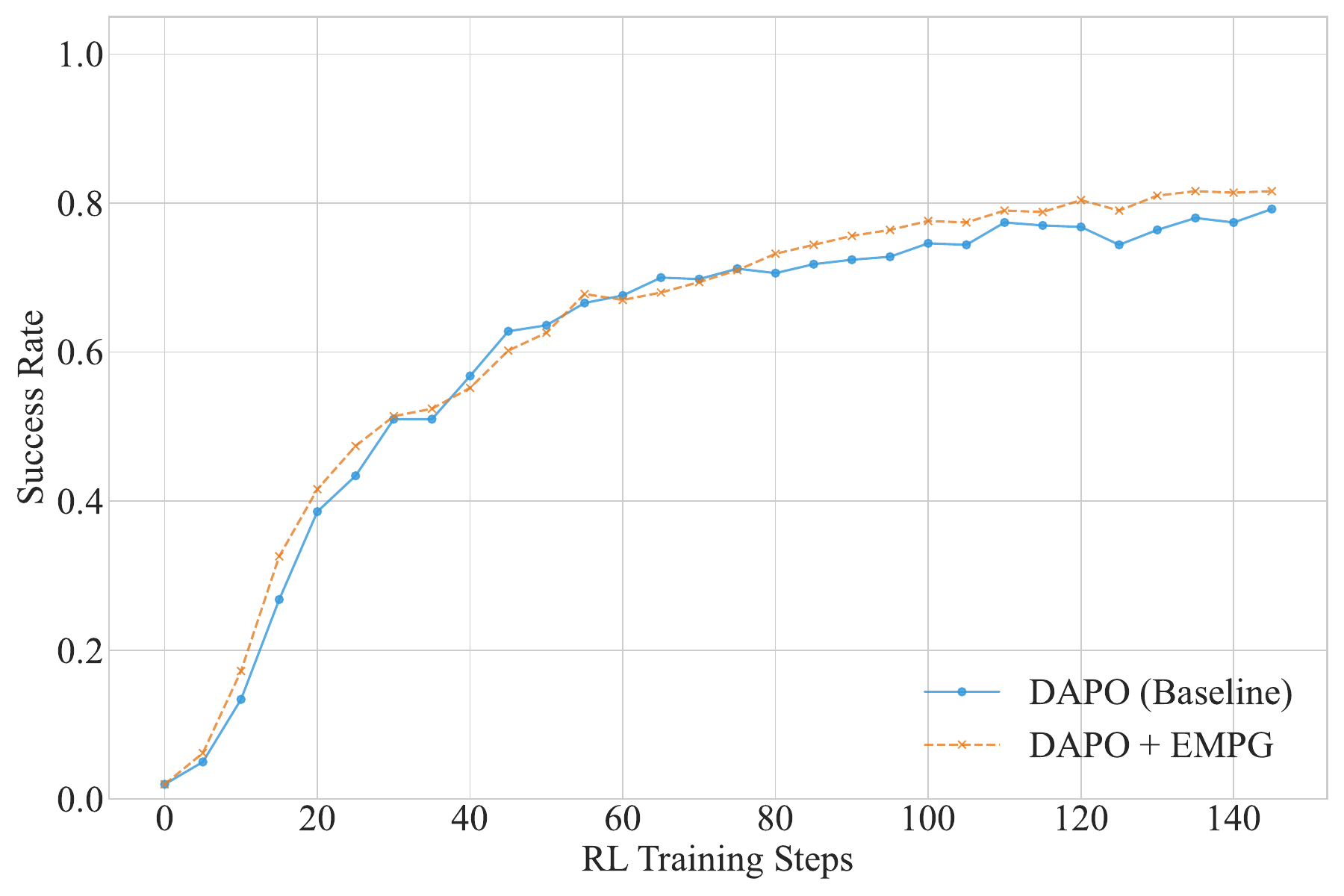}
        \caption{WebShop: DAPO}
        \label{fig:webshop_dapo}
    \end{subfigure}
    
    \vspace{0.3cm} 
    
    \begin{subfigure}[b]{0.48\textwidth}
        \centering
        \includegraphics[width=\linewidth]{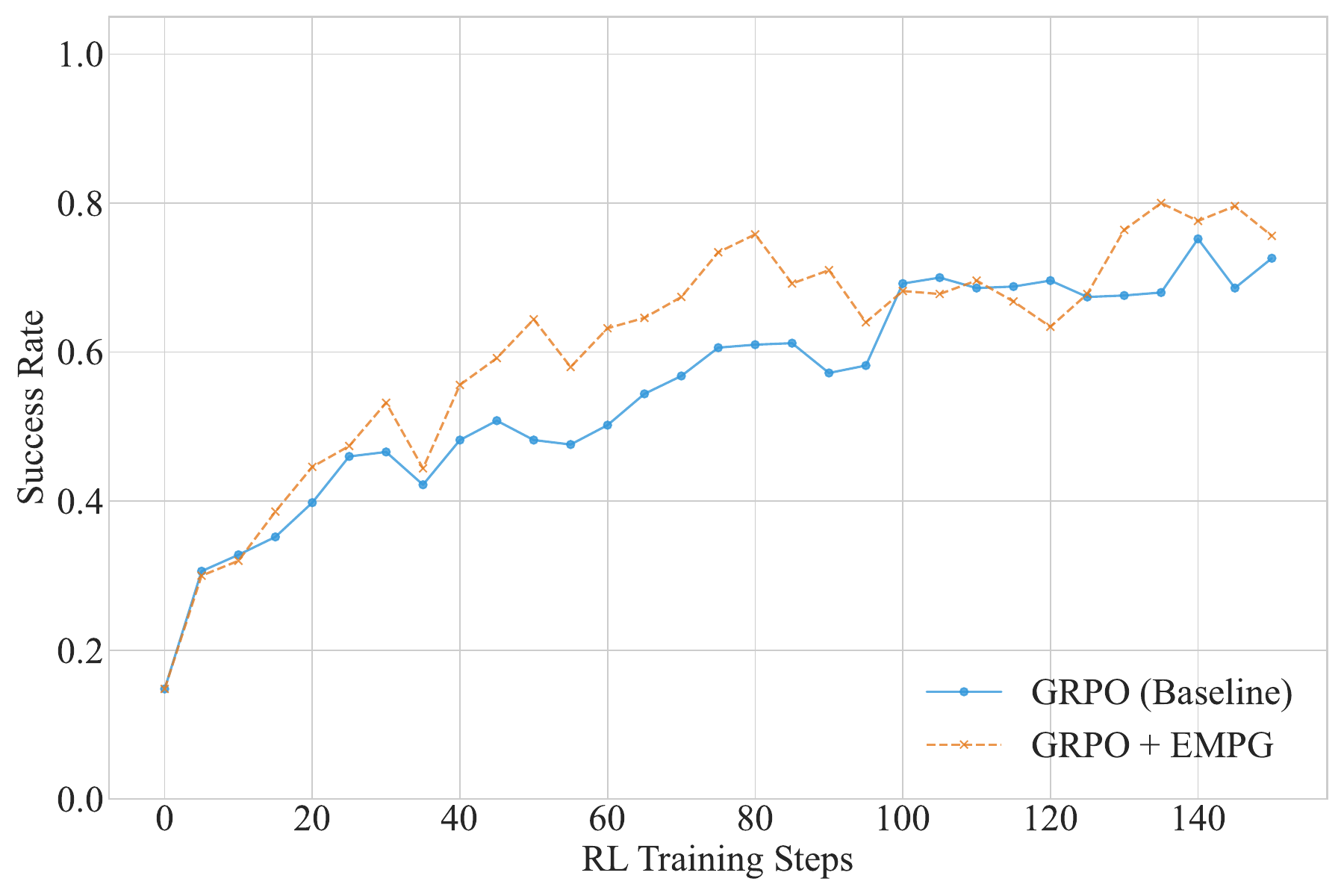}
        \caption{ALFWorld: GRPO}
        \label{fig:alfworld_grpo}
    \end{subfigure}
    \hfill 
    \begin{subfigure}[b]{0.48\textwidth}
        \centering
        \includegraphics[width=\linewidth]{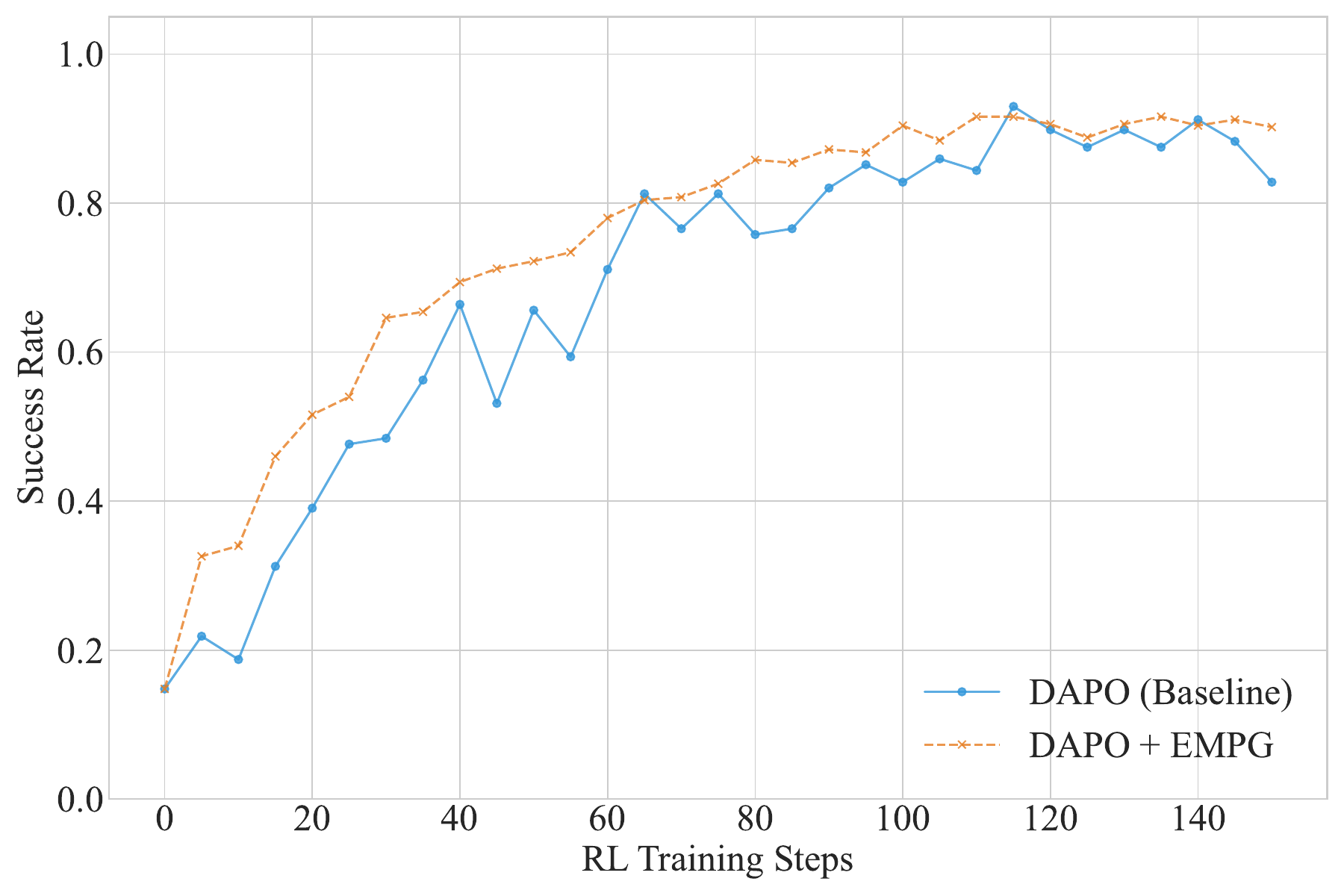}
        \caption{ALFWorld: DAPO}
        \label{fig:alfworld_dapo}
    \end{subfigure}
    
    \caption{
        Learning dynamics comparison for the Qwen2.5-7B-Instruct model on the WebShop and ALFWorld benchmarks (evaluated on the validation set). In all four scenarios, the EMPG-enhanced agents (orange, dashed) demonstrate a superior success rate compared to their respective baselines (blue, solid).
    }
    \label{fig:learning_curves}
\end{figure}

\section{Algorithm Implementation Details}
\label{sec:appendix_implementation}

We provide a PyTorch-style pseudocode implementation for the core logic of our method in Algorithm \ref{alg:empg_code_listings}. This function calculates the final modulated advantage, $A_{\text{final}}$, used for the policy update, as detailed in Section \ref{sec:approach}. The process consists of four main stages:

\begin{enumerate}
    \item \textbf{Step-Level Entropy Collection:} The function first iterates through the batch of trajectories to identify agent action steps (i.e., the ``assistant'' responses). For each step $t$, it computes the corresponding step-level entropy $H_t$ by averaging the policy's token-level entropies for that action.

    \item \textbf{Modulation Component Calculation:} All collected step entropies $\{H_t\}$ are normalized across the batch using min-max scaling to produce $\{H_{\text{norm}, t}\}$ (as per Eq. \ref{eq:h_norm}). These normalized values are then used to compute the two key components of our method: the self-calibrating scaling factor $g(H_t)$ (Eq. \ref{eq:g_func}) and the future clarity bonus term $g'(H_{t+1})$ (Eq. \ref{eq:g_prime_func}).

    \item \textbf{Advantage Modulation:} The function then applies these components to the original outcome-based advantage. For each step, the advantage is scaled by $g(H_t)$ and augmented by the future clarity bonus $\zeta \cdot g'(H_{t+1})$, yielding the modulated advantage $A_{\text{mod}}$ as defined in our main formula (Eq. \ref{eq:main_formula}).

    \item \textbf{Final Normalization:} Finally, to reduce variance and ensure stable training, the entire batch of resulting modulated advantages is normalized to have a mean of zero. This produces the final advantage $A_{\text{final}}$ (Eq. \ref{eq:adv_norm}) that is used to compute the policy gradient.
\end{enumerate}

\begin{algorithm}[t]
    \caption{PyTorch-Style Pseudocode for EMPG Advantage Calculation}
    \label{alg:empg_code_listings}
    
    \begin{spacing}{0.95} 
    \begin{lstlisting}[language=Python, style=compactstyle]
import numpy as np
import torch

def compute_empg_advantage(tokenizer, batch, k=1.0, k_f=1.0, zeta=0.1):
    """
    Args:
        tokenizer: The tokenizer for identifying response segments.
        batch: A data batch with 'responses', 'old_entropy', 'advantages'.
        k (float): Hyperparameter for self-calibrating gradient scaling.
        k_f (float): Hyperparameter for the future clarity bonus.
        zeta (float): Hyperparameter for the future clarity bonus.
    """
    # --- 1. First Pass: Collect Step-Level Entropies ---
    all_step_entropies = []
    # segments_to_modify stores {'sample_idx', 'start', 'end'} for each step
    segments_to_modify = [] 

    for i in range(batch.batch.batch_size[0]):
        # Find "assistant" segments, which correspond to agent steps.
        token_segments = process_token_sequences(
            batch.batch['responses'][i], 
            tokenizer.encode("<|im_start|>assistant\n"), 
            tokenizer.encode('<|im_end|>')
        )
        for start, end in token_segments:
            if start >= end: continue
            
            # Calculate the average token-level entropy for the step
            step_entropy = batch.batch['old_entropy'][i][start:end].mean().item()
            all_step_ entropies.append(step_entropy)
            segments_to_modify.append({'sample_idx': i, 'start': start, 'end': end})

    if not all_step_entropies: return

    # --- 2. Calculate Modulated Advantage Components ---
    H = np.array(all_step_entropies)
    
    # Batch-level entropy normalization (Eq. 12) with \epsilon = 1e-8
    min_H, max_H = np.min(H), np.max(H)
    H_norm = (H - min_H) / (max_H - min_H + 1e-8)

    # Self-calibrating gradient scaling g(H) (Eq. 10)
    g_H_unnormalized = np.exp(-k * H_norm)
    mean_g_H = np.mean(g_H_unnormalized)
    g_H = g_H_unnormalized / (mean_g_H + 1e-8)
    
    # Future clarity bonus f(H) (Eq. 11)
    f_H = np.exp(-k_f * H_norm)

    # Convert to tensors for PyTorch operations
    g_H = torch.tensor(g_H, device=batch.batch['advantages'].device, dtype=torch.float32)
    f_H = torch.tensor(f_H, device=batch.batch['advantages'].device, dtype=torch.float32)

    # --- 3. Second Pass: Apply Advantage Modulation (Eq. 8) ---
    step_advantages = []
    for i, segment in enumerate(segments_to_modify):
        idx, start, end = segment['sample_idx'], segment['start'], segment['end']
        
        # Apply self-calibrating gradient scaling
        batch.batch['advantages'][idx][start:end] *= g_H[i]
        
        # Add future clarity bonus if there is a next step
        next_seg = segments_to_modify[i+1] if i+1 < len(segments_to_modify) else None
        if next_seg and next_seg['sample_idx'] == idx:
            batch.batch['advantages'][idx][start:end] += zeta * f_H[i+1]
        step_advantages.append(batch.batch['advantages'][idx][start])
            
    # --- 4. Final Advantage Normalization (Eq. 7) ---
    if step_advantages:
        final_adv_mean = torch.mean(torch.stack(step_advantages))
        batch.batch['advantages'] -= final_adv_mean
\end{lstlisting}
    \end{spacing}
\end{algorithm}

\end{document}